\documentclass[10pt,english,journal]{IEEEtran}
\usepackage[T1]{fontenc}
\usepackage[utf8]{inputenc}
\usepackage{babel}
\usepackage{refstyle}
\usepackage{mathtools}
\usepackage{bm}
\usepackage{amsmath}
\usepackage{amsthm}
\usepackage{amssymb}
\usepackage{microtype}
\usepackage[pdfusetitle,
 bookmarks=true,bookmarksnumbered=false,bookmarksopen=false,
 breaklinks=false,pdfborder={0 0 1},backref=false,colorlinks=false]
 {hyperref}
\hypersetup{
 hidelinks}

\makeatletter

\AtBeginDocument{\providecommand\secref[1]{\ref{sec:#1}}}
\AtBeginDocument{\providecommand\thmref[1]{\ref{thm:#1}}}
\AtBeginDocument{\providecommand\propref[1]{\ref{prop:#1}}}
\AtBeginDocument{\providecommand\figref[1]{\ref{fig:#1}}}
\AtBeginDocument{\providecommand\defref[1]{\ref{def:#1}}}
\AtBeginDocument{\providecommand\lemref[1]{\ref{lem:#1}}}
\RS@ifundefined{subsecref}
  {\newref{subsec}{name = \RSsectxt}}
  {}
\RS@ifundefined{thmref}
  {\def\RSthmtxt{theorem~}\newref{thm}{name = \RSthmtxt}}
  {}
\RS@ifundefined{lemref}
  {\def\RSlemtxt{lemma~}\newref{lem}{name = \RSlemtxt}}
  {}

\theoremstyle{plain}
\newtheorem{thm}{\protect\theoremname}
\theoremstyle{definition}
\newtheorem{defn}[thm]{\protect\definitionname}
\theoremstyle{remark}
\newtheorem{rem}[thm]{\protect\remarkname}
\theoremstyle{definition}
\newtheorem{example}[thm]{\protect\examplename}
\theoremstyle{plain}
\newtheorem{prop}[thm]{\protect\propositionname}
\theoremstyle{plain}
\newtheorem{lem}[thm]{\protect\lemmaname}
\theoremstyle{remark}
\newtheorem{claim}[thm]{\protect\claimname}

\usepackage{pgfplots}
\pgfplotsset{compat=1.18}	
\usepackage{xcolor}
\definecolor{dantzigcolor}{RGB}{31,119,180}
\definecolor{lassocolor}{RGB}{214,39,40}

\makeatother

\providecommand{\claimname}{Claim}
\providecommand{\definitionname}{Definition}
\providecommand{\examplename}{Example}
\providecommand{\lemmaname}{Lemma}
\providecommand{\propositionname}{Proposition}
\providecommand{\remarkname}{Remark}
\providecommand{\theoremname}{Theorem}

\begin{document}
\title{Sharp Restricted Isometry Thresholds for\\
Global Minima of Rank-Restricted Matrix LASSO}
\author{Richard~Y.~Zhang\thanks{R. Y. Zhang is with the Department of Electrical and Computer Engineering,
University of Illinois Urbana--Champaign, Urbana, IL 61801 USA (e-mail:
ryz@illinois.edu).}\thanks{This work was supported in part by NSF CAREER Award ECCS-2047462 and
ONR Award N00014-24-1-2671.}}
\maketitle
\begin{abstract}
We determine the sharp restricted isometry threshold for recovery
at global minima of the rank-restricted matrix LASSO. For target rank
$r_{\star}$, if the rank-$k$ RIP constant satisfies $\delta<\delta_{\mathrm{sharp}}(k/r_{\star})$,
where $\delta_{\mathrm{sharp}}(t)=t/(4-t)$ for $0<t<4/3$ and $\delta_{\mathrm{sharp}}(t)=\sqrt{(t-1)/t}$
for $t\ge4/3$, then every global minimizer has Frobenius error $\lesssim\sqrt{r_{\star}}\lambda$
for all $\lambda\gtrsim\|\mathcal{A}^{*}(\xi)\|_{\mathrm{op}}$ and
at every search rank $r\ge r_{\star}$. The constants depend only
on the RIP constant and $t=k/r_{\star}$, and in particular are independent
of the search rank. When the rank restriction is inactive, the result
specializes to the ordinary convex matrix LASSO. We also obtain the
analogous results for sparsity-restricted vector LASSO. Conversely,
we show that the threshold $\delta<\delta_{\mathrm{sharp}}(k/r_{\star})$
cannot be improved, due to the existence of counterexamples whose
global minimizers fail to recover the ground truth.
\end{abstract}


\global\long\def\ip#1#2{\left\langle #1,#2\right\rangle }%
\global\long\def\R{\mathbb{R}}%
\global\long\def\A{\mathcal{A}}%
\global\long\def\H{\mathcal{H}}%
\global\long\def\sharpdelta{\delta_{\mathrm{sharp}}}%
\global\long\def\rr{r_{\star}}%
\global\long\def\ss{s_{\star}}%

\global\long\def\RIPm{\operatorname{RIP}}%
\global\long\def\RIPv{\operatorname{rip}}%

\global\long\def\tail{\mathrm{tail}}%
\global\long\def\head{\mathrm{head}}%
\global\long\def\crit{\mathrm{crit}}%

\global\long\def\F{\mathrm{F}}%
\global\long\def\op{\mathrm{op}}%
\global\long\def\nuc{\mathrm{nuc}}%

\global\long\def\norm#1{\left\lVert #1\right\rVert }%
\global\long\def\normnuc#1{\norm{#1}_{\nuc}}%
\global\long\def\normop#1{\norm{#1}_{\op}}%
\global\long\def\normF#1{\norm{#1}_{\F}}%

\global\long\def\vector{\operatorname{vec}}%

\global\long\def\Defect{\operatorname{Defect}}%

\global\long\def\supp{\operatorname{supp}}%
\global\long\def\rank{\operatorname{rank}}%
\global\long\def\spar{\operatorname{spar}}%
\global\long\def\diag{\operatorname{diag}}%
\global\long\def\tr{\operatorname{tr}}%
\global\long\def\cA{\mathsf{A}}%
\global\long\def\cC{\mathsf{C}}%
\global\long\def\cK{\mathsf{K}}%
\global\long\def\cH{\mathsf{H}}%
\global\long\def\cD{\mathsf{D}}%
\global\long\def\cR{\mathsf{R}}%
\global\long\def\cT{\mathsf{T}}%
\global\long\def\cS{\mathsf{S}}%
\global\long\def\cX{\mathsf{X}}%
\global\long\def\cY{\mathsf{Y}}%
\global\long\def\cZ{\mathsf{Z}}%

\global\long\def\argmin{\operatorname*{arg\,min}}%
\global\long\def\Argmin{\operatorname*{Arg\,min}}%
\global\long\def\normtwo#1{\norm{#1}_{2}}%
\global\long\def\normone#1{\norm{#1}_{1}}%
\global\long\def\norminf#1{\norm{#1}_{\infty}}%
\global\long\def\normzero#1{\norm{#1}_{0}}%
\global\long\def\E{\mathbb{E}}%
\global\long\def\Unif{\operatorname{Unif}}%
\global\long\def\e{\mathrm{e}}%

\global\long\def\Dantzig{\mathrm{D}}%
\global\long\def\Lasso{\mathrm{L}}%

\section{Introduction}

\label{sec:intro}Low-rank matrix recovery seeks to estimate an unknown
matrix $X_{\star}\in\R^{d_{1}\times d_{2}}$ of rank at most $\rr$
from $m$ noisy linear measurements 
\begin{equation}
b=\A(X_{\star})+\xi\in\R^{m},\qquad\A:\R^{d_{1}\times d_{2}}\to\R^{m}.\label{eq:model}
\end{equation}
A standard scalable approach fixes a search rank $r\ge\rr$ and fits
a factored model $X=LR^{T}$ with $L\in\R^{d_{1}\times r}$ and $R\in\R^{d_{2}\times r}$
against the regularized least-squares loss 
\begin{equation}
f_{\lambda}(L,R)\coloneqq\normtwo{\A(LR^{T})-b}^{2}+\lambda\bigl(\normF L^{2}+\normF R^{2}\bigr),\label{eq:loss}
\end{equation}
where $\lambda\ge0$ is a regularization parameter. Despite the nonconvexity
of (\ref{eq:loss}), gradient descent on the factors $L$ and $R$
from a random initialization is routinely observed to recover the
ground truth to high accuracy.

Local optimization algorithms can at best guarantee convergence to
a second-order critical point (SOCP), that is, to a local minimum
or to a higher-order saddle point. % TODO: cite the saddle-avoidance result you rely on, e.g. Lee et al. (2016),
% Jin et al. (2017); as written the claim is standard but uncited.
Considerable recent effort has therefore gone into certifying recovery
at every SOCP. As we review in \secref{nonconvex}, the conditions
needed to rule out spurious SOCPs uniformly turn out to be strong,
and they tighten as the search rank grows.

In practice it is widely believed that spurious SOCPs do exist, and
yet that local optimization converges to global optimality anyway,
without any rigorous guarantee of doing so. This paper asks what can
be guaranteed once the optimization problem is solved globally:
\begin{quote}
\emph{Assuming that a lucky algorithm succeeds in finding a global
minimum, what are the fundamental statistical limits in recovering
the underlying ground truth?} 
\end{quote}
Global minimization of (\ref{eq:loss}) is equivalent to the rank-constrained
matrix LASSO 
\begin{equation}
X_{\lambda,r}\in\Argmin_{\begin{subarray}{c}
X\in\R^{d_{1}\times d_{2}}\\
\rank(X)\le r
\end{subarray}}\left\{ \tfrac{1}{2}\normtwo{\A(X)-b}^{2}+\lambda\normnuc X\right\} ,\label{eq:matrix-lasso}
\end{equation}
by way of the variational form of the nuclear norm, 
\begin{equation}
\normnuc X\coloneqq\sum_{i}\sigma_{i}(X)=\min_{LR^{T}=X}\frac{\normF L^{2}+\normF R^{2}}{2}.\label{eq:nuc-identity}
\end{equation}
The same estimator may therefore be read either as nonconvex optimization
over the factors $L,R$, or as nuclear-norm penalization subject to
a rank cap. 

Our main result provides a complete answer to the question above under
the canonical RIP measurement model.
\begin{defn}[Matrix restricted isometry property]
For $\delta\ge0$ and an integer $k\ge1$, we write $\A\in\RIPm(\delta,k)$
if 
\begin{equation}
(1-\delta)\normF E^{2}\le\normtwo{\A(E)}^{2}\le(1+\delta)\normF E^{2}\label{eq:rip}
\end{equation}
holds for every $E$ with $\rank(E)\le k$. 
\end{defn}
\begin{thm}[Rank-restricted matrix LASSO]
\label{thm:lasso} Fix integers $\rr\ge1$ and $k\ge2$. Suppose
that $\rank(X_{\star})\le\rr$ and that $\A\in\RIPm(\delta,k)$. Define
the threshold
\begin{equation}
\sharpdelta(t)\coloneqq\begin{cases}
\sqrt{\dfrac{t-1}{t}}, & t\ge\dfrac{4}{3},\\[2ex]
\dfrac{t}{4-t}, & 0<t<\dfrac{4}{3}.
\end{cases}\label{eq:threshold}
\end{equation}
\begin{itemize}
\item \emph{(Sufficiency)} If $\delta<\sharpdelta(k/\rr)$, then every global
minimizer $X_{\lambda,r}$ of (\ref{eq:matrix-lasso}) with $\lambda>0$
satisfies 
\begin{equation}
r\ge\rr,\;\;\lambda\gtrsim\normop{\A^{*}(\xi)}\implies\normF{X_{\lambda,r}-X_{\star}}\lesssim\sqrt{\rr}\,\lambda,\label{eq:err-bnd}
\end{equation}
where the hidden constants depend only on $\delta$ and on the ratio
$k/\rr$. 
\item \emph{(Necessity)} If $\delta\ge\sharpdelta(k/\rr)+1/\rr$, then there
exist $X_{\star}$ and $\A$ with $\rank(X_{\star})\le\rr$ and $\A\in\RIPm(\delta,k)$
such that, even in the noiseless case $\xi=0$, every global minimizer
$X_{\lambda,r}$ of (\ref{eq:matrix-lasso}) satisfies 
\begin{equation}
r\gtrsim\rr,\;\;\lambda>0\implies\normF{X_{\lambda,r}-X_{\star}}\ge\normF{X_{\star}},\label{eq:err-fail}
\end{equation}
where the hidden constant depends only on the ratio $k/\rr$. 
\end{itemize}
\end{thm}
\begin{rem}[Optimality]
Calibrated at the canonical dual-noise scale $\lambda\asymp\normop{\A^{*}(\xi)}$,
the error bound (\ref{eq:err-bnd}) is minimax optimal up to absolute
constants under Gaussian noise~\cite{CandesPlan2011}. 
\end{rem}
\begin{rem}[Sharpness]
Fix $t>0$ and $\delta>\sharpdelta(t)$. Since $1/\rr\to0$, the
necessity statement applies once $\rr$ is large enough, and it contradicts
the dimension-free bound (\ref{eq:err-bnd}). The condition $\delta<\sharpdelta(t)$
is therefore sharp, in the sense that it cannot be relaxed uniformly
in $\rr$ while retaining (\ref{eq:err-bnd}). The restriction $k\ge2$
is also essential, as shown by the counterexample of Cai and Zhang~\cite[Rem~3.1 and Rem~3.3]{CaiZhang2013}.
\end{rem}
\begin{rem}[Convex relaxation]
Raising $r\ge d\coloneqq\min\{d_{1},d_{2}\}$ so that the rank cap
becomes inactive recovers the ordinary convex matrix LASSO. Because
the constants in (\ref{eq:err-bnd}) do not depend on $r$, the classical
convex estimator inherits the same sharp threshold and the same error
bound. 
\end{rem}
The threshold $\delta<\sharpdelta(t)$ was already known to be sharp
for \emph{constrained} nuclear-norm recovery~\cite{CaiZhang2014,ZhangLi2018}.
For a \emph{penalized} estimator such as the LASSO, however, every
prior result attaining the optimal error bound (\ref{eq:err-bnd})
has done so under a substantially more conservative RIP threshold.
Wang, Zhang, and Wang~\cite{WangZhangWang2021} established the sharp
threshold $\delta<\sharpdelta(t)$ on the high-order branch $t\ge4/3$,
but, as we explain in \secref{wzw}, only at the cost of a suboptimal
error bound. 

To the best of our knowledge, \thmref{lasso} is the first result
to establish the sharp threshold $\delta<\sharpdelta(t)$ together
with the optimal error bound (\ref{eq:err-bnd}). It shows that the
apparent gap between constrained and penalized estimators is an artifact
of proof technique rather than a fundamental limitation of the estimator
itself.

\subsection{Global minima versus the low-rank landscape}

\label{sec:nonconvex}Most of the existing literature on the factored
objective (\ref{eq:loss}) aims to guarantee recovery at every second-order
critical point, that is, at every $(L,R)$ with 
\begin{equation}
\nabla f_{\lambda}(L,R)=0,\qquad\nabla^{2}f_{\lambda}(L,R)\succeq0.\label{eq:socp_def}
\end{equation}
A classical approach, developed by Journée, Bach, Absil, and Sepulchre~\cite{journee2010low},
uses low-rank factorization as a means of reaching global optimality
in the underlying convex problem efficiently. Having computed an SOCP,
one augments the iterate with an all-zero column, $L_{+}=[L,\,0]$
and $R_{+}=[R,\,0]$, and tests whether the lifted point $(L_{+},R_{+})$
still satisfies (\ref{eq:socp_def}). If it does, then $X_{\lambda}=LR^{T}$
is optimal for the convex relaxation and therefore inherits the classical
LASSO guarantees.

The advantage of this incremental lifting scheme is that it provides
a certificate once global optimality has been reached. Its drawback
is that, in principle, the search rank may keep increasing until $r=d$,
which erases much of the computational benefit of factorization. Bhojanapalli,
Boumal, Jain, and Netrapalli~\cite{bhojanapalli2018smoothed} extended
an earlier work of Boumal, Voroninski, and Bandeira~\cite{boumal2016non}
to show that objectives of this type generically have no spurious
local minima once $r>\sqrt{2m}$. Since $m\asymp d$ in the typical
regime, this can still push the number of optimization variables to
$\Theta(d^{3/2})$. 

In practice the landscape often becomes benign at much smaller search
ranks $r=O(1)$. A more refined theory helps explain why incremental
lifting can terminate much sooner. Let $r_{\lambda}$ denote the smallest
rank among the minimizers of the convex matrix LASSO. If $r\ge r_{\lambda}$
and 
\begin{equation}
\delta<\frac{1}{1+\sqrt{r_{\lambda}/r}},\qquad k\ge r+r_{\lambda},\label{eq:global-landscape}
\end{equation}
then every SOCP is \emph{convex-global}, meaning that $LR^{T}$ is
a global minimizer of the convex matrix LASSO~\cite{Zhang2025MP,McRaeZhang2026}.
This guarantee becomes harder to maintain as the search rank grows,
because the required RIP order $k$ grows with $r$. For a fixed sensing
map, (\ref{eq:global-landscape}) may eventually fail as the rank
is lifted, at which point the theorem no longer controls the whole
landscape. Incremental lifting may nevertheless still reach a convex-global
solution; whenever it does, the lifted SOCP test certifies global
optimality directly and \thmref{lasso} applies.

More recent work guarantees recovery directly at every SOCP, without
requiring convex globality. If $r\ge\rr$ and 
\begin{equation}
\delta<\frac{1}{1+\sqrt{\rr/r}},\qquad k\ge r+\rr,\label{eq:direct-landscape}
\end{equation}
then every SOCP yields the optimal rate $O(\sqrt{\rr}\,\lambda)$,
either under the calibration $\lambda\ge\normop{\A^{*}(\xi)}$ or
when the search rank satisfies $r\asymp\rr$~\cite{Zhang2025SIOPT,McRaeZhang2026}.
Here too the required RIP order $k$ grows with the search rank $r$.
Without it, spurious local minima can exist and trap a local method.

Both landscape guarantees require RIP at an order $k$ growing with
the search rank $r$, and both dependencies are sharp. For Gaussian
sensing this costs $m\gtrsim\delta^{-2}r(d_{1}+d_{2})$ measurements,
so overparameterizing the factorization degrades the statistical requirement.
\thmref{lasso} requires RIP only at an order set by $\rr$, and its
error bound $O(\sqrt{\rr}\,\lambda)$ holds at every global minimizer
for every search rank $r\ge\rr$. The price is paid in the hypothesis.
The theorem guarantees recovery at a global minimizer, without telling
us how to find it.

\subsection{Extension to sparse-vector LASSO}

\label{sec:vector}Low-rank matrix recovery has a well-studied predecessor
in sparse vector recovery, and the connection between them is more
than an analogy. The rank of a matrix is the sparsity of its vector
of singular values, and our proof exploits this directly. We first
establish sharp recovery for sparse vectors, and then transport the
resulting geometry to the singular values of the matrix error. The
vector result is therefore both a tool for the matrix argument and
a theorem of independent interest.

In the sparse setting we observe
\begin{equation}
b=Ax_{\star}+\xi\in\R^{m},\qquad A\in\R^{m\times d},\label{eq:vector-model}
\end{equation}
where $x_{\star}$ has at most $\ss$ nonzero entries. Fixing a search
sparsity $s\ge\ss$, the counterpart of (\ref{eq:matrix-lasso}) is
\begin{equation}
x_{\lambda,s}\in\Argmin_{\normzero x\le s}\left\{ \tfrac{1}{2}\normtwo{Ax-b}^{2}+\lambda\normone x\right\} .\label{eq:vector-lasso}
\end{equation}
The two estimators are computationally analogous. Fixing an active
set in (\ref{eq:vector-lasso}) leaves a small convex LASSO, just
as fixing the search rank in (\ref{eq:matrix-lasso}) leaves a low-dimensional
factored problem. Incremental active-set methods enlarge the support
until the sparsity cap is reached or the full LASSO KKT conditions
are certified, exactly as rank-incremental methods lift $r$ until
the rank cap becomes inactive.

Guarantees for (\ref{eq:vector-lasso}) are stated in terms of the
vector analog of the restricted isometry property, which we recall
next. 
\begin{defn}[Vector restricted isometry property]
For $\delta\ge0$ and an integer $k\ge1$, we write $A\in\RIPv(\delta,k)$
if 
\begin{equation}
(1-\delta)\normtwo e^{2}\le\normtwo{Ae}^{2}\le(1+\delta)\normtwo e^{2}\label{eq:vec}
\end{equation}
holds for every $e$ with $\normzero e\le k$. 
\end{defn}
The resulting theorem carries the same threshold $\sharpdelta$ as
\thmref{lasso}, with the sparsity $\ss$ in place of the rank $\rr$.
\begin{thm}[Sparsity-restricted vector LASSO]
\label{thm:lasso-vec} Fix integers $\ss\ge1$ and $k\ge2$. Suppose
that $\normzero{x_{\star}}\le\ss$ and that $A\in\RIPv(\delta,k)$,
and let $\sharpdelta$ be as in (\ref{eq:threshold}). 
\begin{itemize}
\item \emph{(Sufficiency)} If $\delta<\sharpdelta(k/\ss)$, then every global
minimizer $x_{\lambda,s}$ of (\ref{eq:vector-lasso}) with $\lambda>0$
satisfies 
\begin{equation}
s\ge\ss,\;\;\lambda\gtrsim\norminf{A^{T}\xi}\implies\normtwo{x_{\lambda,s}-x_{\star}}\lesssim\sqrt{\ss}\,\lambda,\label{eq:err-bnd-vec}
\end{equation}
where the hidden constants depend only on $\delta$ and on the ratio
$k/\ss$. 
\item \emph{(Necessity)} If $\delta\ge\sharpdelta(k/\ss)+1/\ss$, then there
exist $x_{\star}$ and $A$ with $\normzero{x_{\star}}\le\ss$ and
$A\in\RIPv(\delta,k)$ such that, even in the noiseless case $\xi=0$,
every global minimizer $x_{\lambda,s}$ of (\ref{eq:vector-lasso})
satisfies 
\begin{equation}
s\gtrsim\ss,\;\;\lambda>0\implies\normtwo{x_{\lambda,s}-x_{\star}}\ge\normtwo{x_{\star}},\label{eq:err-fail-vec}
\end{equation}
where the hidden constant depends only on the ratio $k/\ss$. 
\end{itemize}
\end{thm}
Taking $s=d$ removes the sparsity cap and recovers the ordinary convex
vector LASSO; \thmref{lasso-vec} is, to the best of our knowledge,
the first complete sharp-RIP guarantee for that standard estimator.

\subsection{Organization}

Section~II reviews the prior sharp theory for constrained estimators,
the obstruction created by penalization, and the limitations of the
prior result of Wang, Zhang, and Wang~\cite{WangZhangWang2021}.
Section~III proves explicit recovery bounds underlying \thmref{lasso}
and \thmref{lasso-vec}, and Section~IV constructs the matching counterexamples.
Section~V proves the critical \propref{HTR} used in the proof of
the recovery bounds. Section~VI concludes.

\subsection*{Notation}

We write $d\coloneqq\min\{d_{1},d_{2}\}$, and denote the singular
values of $X\in\R^{d_{1}\times d_{2}}$ by $\sigma_{1}(X)\ge\cdots\ge\sigma_{d}(X)$.
The Frobenius, nuclear, and spectral norms are written $\normF{\cdot}$,
$\normnuc{\cdot}$, and $\normop{\cdot}$; for vectors, $\normzero{\cdot}$
counts nonzero entries and $\normone{\cdot}$, $\normtwo{\cdot}$,
$\norminf{\cdot}$ are the usual $\ell_{p}$ norms. The adjoint of
$\A$ is $\A^{*}$, and $\succeq0$ denotes positive semidefiniteness.
We write $u\lesssim v$ if $u\le Cv$ for a constant $C$ depending
only on $\delta$ and on the ratio $k/\rr$ (respectively $k/\ss$),
and $u\asymp v$ if $u\lesssim v\lesssim u$.

\section{Prior work \& Obstructions}

The basic insight behind low-rank recovery is that, under RIP, the
generally intractable search for a minimum-rank solution can be replaced
by tractable nuclear-norm minimization~\cite{RechtFazelParrilo2010}.
In the presence of noise, one enlarges the feasible set to include
matrices that fit the measurements to the prescribed noise level.

Constrained estimators implement this principle directly, by minimizing
the nuclear norm over a noise-calibrated feasible set. The vector
Dantzig selector was introduced by Candès and Tao~\cite{CandesTao2007}.
Candès and Plan~\cite{CandesPlan2011} introduced its matrix counterpart
\begin{equation}
X_{\lambda}^{\Dantzig}\in\argmin_{X\in\R^{d_{1}\times d_{2}}}\left\{ \normnuc X:\normop{\A^{*}(\A(X)-b)}\le\lambda\right\} \label{eq:matrix-dantzig}
\end{equation}
alongside the matrix LASSO, which is (\ref{eq:matrix-lasso}) with
$r=d$. Assuming $(\delta,4\rr)$-RIP and Gaussian noise, they proved
that both attain the minimax Frobenius rate of \thmref{lasso},
together with sharper oracle inequalities, from essentially the minimum
number of measurements the problem allows. 

The Candès--Plan thresholds were $\delta<\sqrt{2}-1\approx0.414$
for the constrained estimator and the more conservative $\delta<(3\sqrt{2}-1)/17\approx0.191$
for the penalized one. Improving these is subtler than it appears,
because the sensing operator $\A$ can satisfy $(\delta_{k},k)$-RIP
at every order $k$, with $\delta_{k}$ nondecreasing in $k$. A guarantee
of the form $\delta_{k}<c_{k}$ is really one member of a family indexed
by $k$. Recovery is certified as soon as any single member of that
family holds. A smaller constant at a smaller order may therefore
be an improvement, so the object of interest is the whole profile
$k\mapsto c_{k}$. For instance, Mohan and Fazel~\cite{MohanFazel2010}
give $\delta_{5\rr}<0.607$, $\delta_{4\rr}<0.558$, $\delta_{3\rr}<0.4721$,
and $\delta_{2\rr}<0.307$, no one of which implies another. See also~\cite{Candes2008RIP,FoucartLai2009,Foucart2010,MoLi2011,WangLi2013,CaiZhang2013}
for the progression of sufficient RIP conditions.

Cai and Zhang~\cite{CaiZhang2014} determined the complete order-dependent
recovery threshold $\delta_{k}<\sharpdelta(k/\rr)$ for the constrained
estimator (\ref{eq:matrix-dantzig}). They proved sufficiency on the
high-order branch $k\ge\frac{4}{3}\rr$, and necessity at every order,
by constructing instances for which recovery fails with $\delta_{k}>\sharpdelta(k/\rr)+1/\rr$.
They conjectured that the same curve governs the remaining branch
$2\le k<\frac{4}{3}\rr$, and Zhang and Li~\cite{ZhangLi2018} proved
the conjecture, closing the question for constrained recovery.

Nothing comparable is known for the penalized estimator, which is
the more convenient of the two to compute because it imposes no hard
constraint. The obstruction is that a penalized estimator trades fit
against nuclear norm, and can return an estimate whose nuclear norm
exceeds that of the ground truth. The following example shows this
happening.

\begin{example}[Penalization can increase the nuclear norm]
\label{exa:subopt} Take $d_{1}=d_{2}=3$, identify $\R^{3\times3}$
with $\R^{9}$, and let 
\begin{gather*}
\A(X)=X+X_{11}\diag\left(-\tfrac{2}{3},\tfrac{2}{3},\tfrac{2}{3}\right)\quad\text{for all }X\in\R^{3\times3},\\
X_{\star}=\diag(21,0,0),\qquad\xi=\diag(-7,1,1).
\end{gather*}
Then $\A\in\RIPm(\delta,k)$ with $\delta=\frac{2}{3}$ and $k=2$,
because every matrix $X$ of rank at most $2$ satisfies 
\begin{align*}
\left|\normtwo{\A(X)}^{2}-\normF X^{2}\right| & =\tfrac{4}{3}\left|X_{11}(X_{22}+X_{33})\right|\\
 & \le\tfrac{1}{3}\left(\left|X_{11}\right|+\left|X_{22}+X_{33}\right|\right)^{2}\\
 & \le\tfrac{1}{3}\normnuc X^{2}\le\tfrac{2}{3}\normF X^{2}.
\end{align*}
Since $\rr=1$, we have $\delta<\sharpdelta(k/\rr)=\frac{1}{\sqrt{2}}$,
so \thmref{lasso} applies. With the same $\lambda>0$, the constrained
and penalized estimators are 
\begin{align*}
X_{\lambda}^{\Dantzig} & =\diag\begin{bmatrix}21\lambda-\frac{45}{2}(\lambda-1)_{+}+\frac{1}{2}(\lambda-5)_{+}+(\lambda-20)_{+}\\
15(1-\lambda)_{+}\\
15(1-\lambda)_{+}
\end{bmatrix},\\
X_{\lambda}^{\Lasso} & =\diag\begin{bmatrix}3\lambda-4(\lambda-5)_{+}+(\lambda-20)_{+}\\
3(5-\lambda)_{+}\\
3(5-\lambda)_{+}
\end{bmatrix}.
\end{align*}
Here $\normop{\A^{*}(\xi)}=1$. As shown in \figref{l1-example},
$\normnuc{X_{\lambda}^{\Lasso}}>\normnuc{X_{\star}}$ on the whole
range $\normop{\A^{*}(\xi)}<\lambda<3\normop{\A^{*}(\xi)}$.\hfill{}$\square$ 
\end{example}
\begin{figure}[t]
\centering

\noindent\begin{minipage}[t]{1\linewidth}%
\centering \begin{tikzpicture}
\begin{axis}[
	set layers=standard,
    width=\linewidth,
    height=0.4\linewidth,
    xmin=0, xmax=7,
    ymin=0, ymax=32/21,
    xlabel={$\lambda / \|\A^*(\xi)\|_{\op}$},
	xlabel style={
	    at={(axis description cs:0.93,-0.04)},
	    anchor=north
	},
    title={$\displaystyle \frac{\| X_\lambda - X_\star\|_{\F}}{\|X_\star\|_{\F}}$},
	title style={
	    xshift=3ex,
	    yshift=-6ex
	},
    xtick={0,1,5,20},
    ytick={0,2/7,1,10/7},
	yticklabels={$0$,$2/7$,$1$,$10/7$},
    axis lines=left,
    legend style={
        draw=none,
        fill=none,
        font=\small,
        at={(0.98,0.98)},
        anchor=north east
    },
    samples=100,
]

\begin{pgfonlayer}{axis background}
% Shade the regime 0 <= lambda <= 1
\fill[gray!20]
    (axis cs:0,0) rectangle (axis cs:1,32/21);

% Shade the regime 1 <= lambda <= 3
\fill[red!5]
    (axis cs:1,0) rectangle (axis cs:3,32/21);
\end{pgfonlayer}

% Dantzig selector
\addplot[
    very thick,
    dantzigcolor,
    domain=0:1
]
{3*sqrt(99)*(1-x)/21};
\addplot[
    very thick,
    dantzigcolor,
    domain=1:5,
    forget plot
]
{3*(x-1)/(2*21)};
\addplot[
    very thick,
    dantzigcolor,
    domain=5:20,
    forget plot
]
{(x+1)/21};
\addplot[
    very thick,
    dantzigcolor,
    domain=20:21,
    forget plot
]
{1};
\addlegendentry{$X^\Dantzig_\lambda$}

% LASSO
\addplot[
    very thick,
    lassocolor,
    domain=0:5
]
{3*sqrt(3*x^2-34*x+99)/21};
\addplot[
    very thick,
    lassocolor,
    domain=5:20,
    forget plot
]
{(x+1)/21};
\addplot[
    very thick,
    lassocolor,
    domain=20:21,
    forget plot
]
{1};
\addlegendentry{$X^\Lasso_\lambda$}

\addplot[
    thin,
	gray!50,
    solid,
    domain=0:21
]
{2/7};

\end{axis}
\end{tikzpicture}%
\end{minipage}

\noindent\begin{minipage}[t]{1\linewidth}%
\centering \begin{tikzpicture}
\begin{axis}[
	set layers=standard,
    width=\linewidth,
    height=0.4\linewidth,
    xmin=0, xmax=7,
    ymin=10/21, ymax=32/21,
    xlabel={$\lambda / \|\A^*(\xi)\|_{\op}$},
	xlabel style={
	    at={(axis description cs:0.93,-0.04)},
	    anchor=north
	},
    title={$\displaystyle \frac{\|X_\lambda\|_{\nuc}}{\|X_\star\|_{\nuc}}$},
	title style={
		xshift=+3ex,
	    yshift=-6ex
	},
    xtick={0,1,3,5,20},
    ytick={5/7,1,10/7},
	yticklabels={$5/7$,$1$,$10/7$},
    axis lines=left,
    legend style={
        draw=none,
        fill=none,
        font=\small,
        at={(0.98,0.98)},
        anchor=north east
    },
]

\begin{pgfonlayer}{axis background}
\fill[gray!20]
% Shade the regime 0 <= lambda <= 1
    (axis cs:0.02,10.3/21) rectangle (axis cs:1,32/21);
% Shade the regime 1 <= lambda <= 3
\fill[red!5]
    (axis cs:1,10.3/21) rectangle (axis cs:3,32/21);
\end{pgfonlayer}

% Important grid lines
\addplot[
    thin,
    solid,
    domain=0:21
]
{1};
\addplot[
    thin,
	gray!50,
    solid,
    domain=0:21
]
{5/7};

% Dantzig selector
\addplot[
    very thick,
    dantzigcolor,
    domain=0:1
]
{(30-9*x)/21};
\addplot[
    very thick,
    dantzigcolor,
    domain=1:5,
    forget plot
]
{(45-3*x)/(2*21)};
\addplot[
    very thick,
    dantzigcolor,
    domain=5:20,
    forget plot
]
{(20-x)/21};
\addplot[
    very thick,
    dantzigcolor,
    domain=20:21,
    forget plot
]
{0};
%\addlegendentry{$\widehat x_S$}

% LASSO
\addplot[
    very thick,
    lassocolor,
    domain=0:5
]
{(30-3*x)/21};
\addplot[
    very thick,
    lassocolor,
    domain=5:20,
    forget plot
]
{(20-x)/21};
\addplot[
    very thick,
    lassocolor,
    domain=20:21,
    forget plot
]
{0};
%\addlegendentry{$\widehat x_L$}

\end{axis}
\end{tikzpicture}%
\end{minipage}

\caption{\protect\label{fig:l1-example}Relative recovery error (top) and nuclear
norm (bottom) for Example~\ref{exa:subopt}. Gray shading marks $0\le\lambda<\protect\normop{\protect\A^{*}(\xi)}$,
where the ground truth is not Dantzig-feasible. Red shading marks
$\protect\normop{\protect\A^{*}(\xi)}<\lambda<3\protect\normop{\protect\A^{*}(\xi)}$,
where the LASSO nuclear norm exceeds $\protect\normnuc{X_{\star}}$.}
\end{figure}
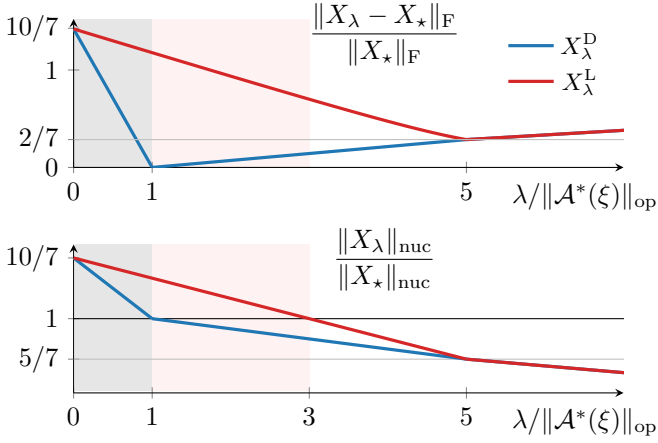

Any analysis of a penalized method must therefore quantify what is
lost when nuclear-norm optimality fails. A constrained estimator pays
nothing here, because its nuclear-norm optimality gives $\normnuc{X_{\lambda}^{\Dantzig}}\le\normnuc{X_{\star}}$
for free. The LASSO proof must instead measure the failure of that
inequality and carry the residue through the RIP argument. 

Can the extension be done without giving up either the optimal error
rate or the sharp RIP threshold? Existing arguments give up one or
the other. The standard penalized proof, reviewed in \secref{stnd},
keeps the optimal rate and sacrifices sharpness. Wang, Zhang, and
Wang~\cite{WangZhangWang2021} recover the sharp high-order threshold
but, as \secref{wzw} explains, at the cost of a statistically suboptimal
dependence on the noise. 

\subsection{Why the Standard Proof Loses Sharpness}

\label{sec:stnd}

The loss of sharpness is easiest to see when benchmarked against the
standard Dantzig-selector argument. Write $E=X_{\lambda}^{\Dantzig}-X_{\star}$
and split its singular values into a head and a tail, 
\[
\cH=\left(\sum_{i\le\rr}\sigma_{i}^{2}(E)\right)^{1/2},\qquad\cT=\frac{1}{\sqrt{\rr}}\sum_{i>\rr}\sigma_{i}(E),
\]
so that $\cH$ captures the part of the error concentrated on the
top $\rr$ directions and $\cT$ the part spread across the remaining
ones. The argument rests on two facts, 
\begin{align}
\text{(feasibility)}\quad & \normop{\A^{*}(\A(X_{\lambda}^{\Dantzig})-b)}\le\lambda,\label{eq:feas}\\
\text{(optimality)}\quad & \normnuc{X_{\lambda}^{\Dantzig}}\le\normnuc{X_{\star}}.\label{eq:minnuke}
\end{align}
Feasibility (\ref{eq:feas}), combined with RIP and the calibration
$\lambda\ge\normop{\A^{*}(\xi)}$, bounds the head by the tail and
the noise level, 
\begin{equation}
\cH\le\rho_{\delta,t}\,\cT+\omega_{\delta,t}\,\lambda\sqrt{\rr},\label{eq:feas-bnd}
\end{equation}
for constants $\rho_{\delta,t}$ and $\omega_{\delta,t}$ depending
only on $\delta$ and $t$. Optimality (\ref{eq:minnuke}) runs in
the opposite direction and bounds both the tail and the recovery error
by the head, 
\begin{equation}
\cT\le\cH,\qquad\normF{X_{\lambda}^{\Dantzig}-X_{\star}}\le\sqrt{2}\,\cH.\label{eq:minnuke-bnd}
\end{equation}
Together, (\ref{eq:feas-bnd}) and (\ref{eq:minnuke-bnd}) close a
scalar feedback loop and yield 
\[
\normF{X_{\lambda}^{\Dantzig}-X_{\star}}\le\frac{\sqrt{2}\,\omega_{\delta,t}}{1-\rho_{\delta,t}}\,\lambda\sqrt{\rr}.
\]
The optimal rate $O(\lambda\sqrt{\rr})$ therefore holds as soon as
the loop gain satisfies $\rho_{\delta,t}<1$, and in the sharp Dantzig-selector
analysis this is exactly the condition $\delta<\sharpdelta(t)$; for
$t\ge4/3$, see \cite[Prop.~3.1, eq.~(16)]{CaiZhang2013}. The threshold
is thus decided by a single gain constant.

Now consider the ordinary convex matrix LASSO, set $E=X_{\lambda}^{\mathrm{L}}-X_{\star}$,
and define $\cH$ and $\cT$ as before. Its KKT condition implies
the dual-residual bound~(\ref{eq:feas}), but the nuclear-norm comparison
(\ref{eq:minnuke}) can fail. Under $\lambda\ge2\|\A^{*}(\xi)\|_{\op}$,
the standard basic inequality supplies only the enlarged form (see~\cite[Cor.~B.2]{BickelRitovTsybakov2009}
or~\cite[Sec.~3.7]{CandesPlan2011}) 
\begin{equation}
\cT\le3\cH,\qquad\normF{X_{\lambda,r}^{\Lasso}-X_{\star}}\le2\cH.\label{eq:notmin-bnd}
\end{equation}
The same feedback computation then gives 
\[
\normF{X_{\lambda,r}^{\Lasso}-X_{\star}}\le\frac{2\,\omega_{\delta,t}}{1-3\rho_{\delta,t}}\,\lambda\sqrt{\rr},
\]
which recovers the optimal rate only when $\rho_{\delta,t}<1/3$.
The factor-three enlargement leaves the statistical rate intact and
inflates the loop gain threefold, and it is this inflation alone that
costs the sharp RIP threshold.

To our knowledge, every prior RIP-based analysis that transfers constrained
recovery arguments to the penalized estimator proceeds through some
version of the multiplicative enlargement (\ref{eq:notmin-bnd});
see, e.g., \cite[Cor.~B.2]{BickelRitovTsybakov2009}, \cite[Sec.~3.7]{CandesPlan2011},
and \cite[Lem.~1(b), eq.~(32)]{NegahbanWainwright2011}. And by Example~\ref{exa:subopt},
sharpness cannot be recovered by restoring (\ref{eq:minnuke}), since
a LASSO solution really can carry more nuclear norm than the ground
truth. What remains open is whether a different argument is able to
reach the sharp threshold.

\subsection{Wang--Zhang--Wang: Sharp High-Order RIP with Euclidean-Noise Calibration}

\label{sec:wzw}

The closest prior result is due to Wang, Zhang, and Wang~\cite{WangZhangWang2021}.
After rescaling their regularized nuclear-norm objective to the normalization
in (\ref{eq:matrix-lasso}), their high-order theorem has the following
consequence.
\begin{thm}[{{\cite[Thm.~4]{WangZhangWang2021}}}]
\label{thm:wzw} Fix integers $k>\rr\ge1$ and write $t=k/\rr$.
Suppose $\rank(X_{\star})\le\rr$ and $\A\in\RIPm(\delta,k)$. If
\begin{equation}
\delta<\sqrt{\frac{t-1}{t}},\label{eq:wzw-thresh}
\end{equation}
then every convex matrix-LASSO solution satisfies 
\begin{equation}
\normF{X_{\lambda}-X_{\star}}\lesssim\sqrt{\rr}\,\lambda+\normtwo{\xi}+\frac{\normtwo{\xi}^{2}}{\sqrt{\rr}\,\lambda},\label{eq:wzw-bound}
\end{equation}
where the hidden constant depends only on $\delta$ and $t$. 
\end{thm}
Condition (\ref{eq:wzw-thresh}) agrees with the sharp threshold $\delta<\sharpdelta(t)$
only on the high-order branch $t\ge4/3$; for $1<t<4/3$ it is strictly
more conservative. The more serious issue is that the error bound
(\ref{eq:wzw-bound}) is driven by the Euclidean norm $\normtwo{\xi}$
of the noise rather than by the dual norm $\normop{\A^{*}(\xi)}$.
Under Gaussian sensing, the two scale differently, and the resulting
guarantee does not improve with the number of measurements.
\begin{example}[Gaussian sensing and Euclidean-noise scaling]
\label{ex:wzw-gaussian} Suppose we observe 
\[
y_{i}=\ip{G_{i}}{X_{\star}}+\varepsilon_{i},\qquad\varepsilon_{i}\overset{\mathrm{i.i.d.}}{\sim}\mathcal{N}(0,\sigma^{2}),
\]
where the $G_{i}$ have independent standard Gaussian entries. Normalizing
turns this into an RIP recovery problem, 
\[
\A(X)=\frac{1}{\sqrt{m}}\bigl[\ip{G_{i}}X\bigr]_{i=1}^{m},\qquad\xi\sim\mathcal{N}\!\left(0,\frac{\sigma^{2}}{m}I_{m}\right),
\]
and with high probability $\A\in\RIPm(\delta,k)$ once $m\gtrsim\delta^{-2}k(d_{1}+d_{2})$.
The bound (\ref{eq:wzw-bound}) is minimized at $\lambda\asymp\normtwo{\xi}/\sqrt{\rr}$,
where it gives 
\[
\normF{X_{\lambda}-X_{\star}}\lesssim\normtwo{\xi}\asymp\sigma.
\]
For fixed $d_{1}$, $d_{2}$, and $\rr$, this stays of constant order
as $m\to\infty$ and so certifies no consistency under oversampling.
The dual-noise calibration $\lambda\asymp\normop{\A^{*}(\xi)}$ in
(\ref{eq:err-bnd}) instead gives, with high probability, 
\[
\normF{X_{\lambda}-X_{\star}}\lesssim\sqrt{\rr}\,\lambda\asymp\sigma\sqrt{\frac{\rr(d_{1}+d_{2})}{m}},
\]
which carries the minimax dependence on the rank, the ambient dimensions,
and the number of measurements~\cite{CandesPlan2011}.\hfill{}$\square$ 
\end{example}

\section{\protect\label{sec:main}LASSO recovery guarantees}

We now state the deterministic recovery bounds with explicit constants.
Throughout, $r_{\star}$ denotes the intrinsic rank or sparsity, and
$r$ denotes the corresponding search size. For $t>0$ and $0\le\delta<\sharpdelta(t)$,
define\begin{subequations}\label{eq:rhotau}
\begin{equation}
\rho_{\delta,t}:=\begin{cases}
{\displaystyle \sqrt{\frac{\delta}{t-(3-t)\delta}},} & 0<t<4/3,\\[2.2ex]
{\displaystyle \frac{\delta}{\sqrt{(1-\delta^{2})(t-1)}},} & t\ge4/3,
\end{cases}\label{eq:rho}
\end{equation}
 and 
\begin{equation}
\tau_{\delta,t}:=\begin{cases}
{\displaystyle \frac{\max\{t,\sqrt{t}\}\sqrt{1+\delta}}{t-(3-t)\delta},} & 0<t<4/3,\\[2.4ex]
{\displaystyle \frac{2}{(1-\delta)\sqrt{1+\delta}},} & t\ge4/3.
\end{cases}\label{eq:tau}
\end{equation}
\end{subequations}On both branches,
\[
\delta<\sharpdelta(t)\quad\Longleftrightarrow\quad\rho_{\delta,t}<1.
\]
Accordingly, $1-\rho_{\delta,t}$ is the contraction margin in the
head--tail feedback argument, whereas $\tau_{\delta,t}$ scales the
contribution of the measured error. This separation makes the dependence
on the RIP constant explicit.
\begin{thm}
\label{thm:matrix-main}Fix integers $r_{\star}\ge1,$ $k\ge2$, and
set $t\coloneqq k/\rr$. Suppose $\rank(X_{\star})\le r_{\star}$
and that $\A\in\RIPm(\delta,k)$. If $\delta<\sharpdelta(t)$ and
the regularization parameter $\lambda>0$ satisfies
\begin{equation}
\lambda\ge2\frac{1+\rho_{\delta,t}}{1-\rho_{\delta,t}}\normop{\A^{*}(\xi)},\label{eq:matrix-calibration-main}
\end{equation}
then for any search rank $r\ge\rr$, the solution $X_{\lambda,r}$
in (\ref{eq:matrix-lasso}) satisfies\begin{subequations}
\begin{align}
\normtwo{\A(X_{\lambda,r}-X_{\star})} & \le\frac{4\tau_{\delta,t}}{1-\rho_{\delta,t}}\sqrt{\rr}\,\lambda,\label{eq:matrix-prediction-main}\\
\normF{X_{\lambda,r}-X_{\star}} & \le\frac{16\tau_{\delta,t}^{2}}{(1-\rho_{\delta,t})^{2}}\sqrt{\rr}\,\lambda.\label{eq:matrix-error-main}
\end{align}
\end{subequations}
\end{thm}
We first establish the vector result. The matrix theorem then follows
by applying the same argument to the singular values of $X_{\lambda,r}-X_{\star}$.
Under this reduction, sparsity becomes rank, the $\ell_{1}$ and $\ell_{2}$
norms become the nuclear and Frobenius norms, and matrix RIP induces
vector RIP along the singular directions of the error.
\begin{thm}
\label{thm:vector-main}Fix integers $r_{\star}\ge1,$ $k\ge2$, and
set $t\coloneqq k/\rr$. Suppose $\normzero{x_{\star}}\le r_{\star}$
and that $A\in\RIPv(\delta,k)$. If $\delta<\sharpdelta(t)$ and the
regularization parameter $\lambda>0$ satisfies
\begin{equation}
\lambda\ge2\frac{1+\rho_{\delta,t}}{1-\rho_{\delta,t}}\norminf{A^{T}\xi},\label{eq:vector-calibration-main}
\end{equation}
then for any sparsity level $r\ge\rr$, the solution $x_{\lambda,r}$
in (\ref{eq:vector-lasso}) satisfies \begin{subequations}
\begin{align}
\normtwo{A(x_{\lambda,r}-x_{\star})} & \le\frac{4\tau_{\delta,t}}{1-\rho_{\delta,t}}\sqrt{\rr}\,\lambda,\label{eq:vector-prediction-main}\\
\normtwo{x_{\lambda,r}-x_{\star}} & \le\frac{16\tau_{\delta,t}^{2}}{(1-\rho_{\delta,t})^{2}}\sqrt{\rr}\,\lambda.\label{eq:vector-error-main}
\end{align}
\end{subequations}
\end{thm}
The proof begins with the classical decomposition of the error magnitudes
into a sparse, spiky head and a dense, diffuse tail. The head is measured
in $\ell_{2}$, while the tail is measured through scale-normalized
$\ell_{1}$ and $\ell_{\infty}$ quantities. The key additional variable
is the \emph{head--tail defect}, which records the amount by which
the error violates the exact $\ell_{1}$ head--tail relation used
in constrained recovery.
\begin{defn}[Head, tail, and defect]
\label{def:HTD} Given $\sigma\in\R^{d}$ and $\rr\le d$, define
\[
\sigma_{\head}\in\argmin_{\normzero{\tilde{\sigma}}\le\rr}\normone{\sigma-\tilde{\sigma}},\quad\sigma_{\tail}:=\sigma-\sigma_{\head},
\]
and set\begin{subequations}\label{eq:HTD}
\begin{align}
\mathsf{H} & :=\normtwo{\sigma_{\head}},\label{eq:H}\\
\mathsf{T} & :=\max\left\{ \frac{\normone{\sigma_{\tail}}}{\sqrt{\rr}},\sqrt{\rr}\norminf{\sigma_{\tail}}\right\} ,\label{eq:T}\\
\mathsf{D} & :=\frac{\bigl(\normone{\sigma_{\tail}}-\normone{\sigma_{\head}}\bigr)_{+}}{\sqrt{\rr}}.\label{eq:D}
\end{align}
\end{subequations}
\end{defn}
The following proposition contains the measurement geometry needed
for both estimators.
\begin{prop}[Sharp head--tail estimate]
\label{prop:HTR}Fix $k\ge2$ and $\rr\ge1$, and let $t\coloneqq k/\rr$.
Suppose $A\in\RIPv(\delta,k)$ with $\delta<\sharpdelta(t)$. For
$\sigma\in\R^{d}$, the quantities $\cH$ and $\cT$ defined in \defref{HTD}
satisfy
\begin{equation}
\cH\le\rho_{\delta,t}\,\cT+\tau_{\delta,t}\,\normtwo{A\sigma}\label{eq:HTR}
\end{equation}
where $\rho_{\delta,t}$ and $\tau_{\delta,t}$ are defined in (\ref{eq:rhotau}).
\end{prop}
The high-order branch $t\ge4/3$ is precisely the estimate used by
Wang, Zhang, and Wang~\cite[Lem.~2]{WangZhangWang2021} in their
high-order argument. The low-order branch $0<t<4/3$ is new and is
obtained by carefully modifying the sharp proof of Zhang and Li~\cite{ZhangLi2018}
for constrained recovery. A complete proof of both branches appears
in \secref{proof-proposition}.

We can now explain why our argument is able to preserve sharpness
for the LASSO. The defect controls the $\ell_{1}$ component of $\cT$,
while the best-$\rr$ ordering controls its $\ell_{\infty}$ component:
\begin{align*}
\frac{\normone{\sigma_{\tail}}}{\sqrt{\rr}} & \le\frac{\normone{\sigma_{\head}}}{\sqrt{\rr}}+\cD\le\cH+\cD,\\
\sqrt{\rr}\norminf{\sigma_{\tail}} & \le\frac{\normone{\sigma_{\head}}}{\sqrt{\rr}}\le\cH.
\end{align*}
Consequently, the head and defect together control the tail:
\begin{equation}
\cT\le\cH+\cD.\label{eq:T-HD}
\end{equation}
This bookkeeping is elementary but decisive. Substituting (\ref{eq:T-HD})
into \propref{HTR} and eliminating the feedback on $\cH$ gives
\begin{equation}
\mathsf{H}\le\frac{\tau_{\delta,t}}{1-\rho_{\delta,t}}\normtwo{A\sigma}+\frac{\rho_{\delta,t}}{1-\rho_{\delta,t}}\mathsf{D}.\label{eq:H-before-D}
\end{equation}
Observe that the sharp threshold $\rho_{\delta,t}<1$ has been maintained,
because the coefficient of $\cH$ in (\ref{eq:T-HD}) is exactly one.
Replacing (\ref{eq:T-HD}) by the coarse bound $\cT\le3\cH$ would
instead require $3\rho_{\delta,t}<1$ and lose sharpness.

The fit error $\normtwo{A\sigma}$ in (\ref{eq:H-before-D}) is a
``noise'' term, that is in turn controlled by $\cH$ through estimator
optimality. The lemma below makes this explicit, while also showing
how $\cH$ further controls other ingredients. Its key prerequisite
(\ref{eq:basic}) is a basic consequence of the optimality of $x_{\lambda,r}$.
We verify this interface separately for the two estimators after completing
the abstract closure argument.
\begin{lem}
\label{lem:head-controls}Let $1\le\rr\le d$, let $A\in\RIPv(\delta,k)$
with $\delta<\sharpdelta(k/\rr)$, and let $\lambda>\lambda_{0}\ge0$.
 Suppose $\sigma\in\R^{d}$ satisfies 
\begin{equation}
\frac{1}{2}\normtwo{A\sigma}^{2}+\lambda\bigl(\normone{\sigma_{\tail}}-\normone{\sigma_{\head}}\bigr)\le\lambda_{0}\normone{\sigma}.\label{eq:basic}
\end{equation}
Then \begin{subequations}\label{eq:head-controls}
\begin{align}
\mathsf{D} & \le\frac{2\lambda_{0}}{\lambda-\lambda_{0}}\mathsf{H},\label{eq:D-control}\\
\normtwo{A\sigma}^{2} & \le2(\lambda+\lambda_{0})\sqrt{\rr}\,\mathsf{H},\label{eq:residual-control}\\
\normtwo{\sigma} & \le\sqrt{2+\frac{2\lambda_{0}}{\lambda-\lambda_{0}}}\,\mathsf{H}.\label{eq:error-control}
\end{align}
\end{subequations}
\end{lem}
\begin{IEEEproof}
Rearrange (\ref{eq:basic}) as 
\begin{equation}
\frac{1}{2}\normtwo{A\sigma}^{2}+(\lambda-\lambda_{0})\normone{\sigma_{\tail}}\le(\lambda+\lambda_{0})\normone{\sigma_{\head}}.\label{eq:basic-rearranged}
\end{equation}
Dropping the tail term and using $\normone{\sigma_{\head}}\le\sqrt{\rr}\,\mathsf{H}$
gives (\ref{eq:residual-control}).  If $\mathsf{D}>0$, then $\normone{\sigma_{\tail}}=\normone{\sigma_{\head}}+\sqrt{\rr}\,\mathsf{D}$.
Substitute this identity into (\ref{eq:basic-rearranged}), drop the
measured-residual term, and again use $\normone{\sigma_{\head}}\le\sqrt{\rr}\,\mathsf{H}$
to obtain (\ref{eq:D-control}); when $\mathsf{D}=0$ the same inequality
is immediate. For the total error, ordering gives 
\begin{align*}
\normtwo{\sigma_{\tail}}^{2} & \le\norminf{\sigma_{\tail}}\normone{\sigma_{\tail}}\\
 & \le\frac{\normone{\sigma_{\head}}}{\rr}\bigl(\normone{\sigma_{\head}}+\sqrt{\rr}\,\mathsf{D}\bigr)\\
 & \le\mathsf{H}^{2}+\mathsf{D}\mathsf{H}.
\end{align*}
Therefore $\normtwo{\sigma}^{2}\le2\mathsf{H}^{2}+\mathsf{D}\mathsf{H}$;
substituting (\ref{eq:D-control}) proves (\ref{eq:error-control}).
\end{IEEEproof}
Combining~(\ref{eq:H-before-D}) with (\ref{eq:head-controls}) closes
the feedback among $\cH$, $\cD$, $\normtwo{A\sigma}$, and $\normtwo{\sigma}$.
\begin{lem}
\label{lem:closure}Under the same conditions as \lemref{head-controls}.
If 
\begin{equation}
\lambda\ge2\frac{1+\rho_{\delta,t}}{1-\rho_{\delta,t}}\lambda_{0},\label{eq:abstract-calibration}
\end{equation}
 then \begin{subequations}
\begin{align}
\normtwo{A\sigma} & \le\frac{4\tau_{\delta,t}}{1-\rho_{\delta,t}}\sqrt{\rr}\,\lambda,\label{eq:abstract-residual}\\
\normtwo{\sigma} & \le\frac{16\tau_{\delta,t}^{2}}{(1-\rho_{\delta,t})^{2}}\sqrt{\rr}\,\lambda.\label{eq:abstract-error}
\end{align}
\end{subequations}
\end{lem}
\begin{IEEEproof}
Because $\lambda>\lambda_{0}$, Lemma~\ref{lem:head-controls} applies.
The calibration (\ref{eq:abstract-calibration}) implies \begin{subequations}
\begin{align}
\frac{2\lambda_{0}}{\lambda-\lambda_{0}} & \le\frac{2(1-\rho_{\delta,t})}{1+3\rho_{\delta,t}},\label{eq:calib-D}\\
2(\lambda+\lambda_{0}) & \le\frac{3+\rho_{\delta,t}}{1+\rho_{\delta,t}}\lambda.\label{eq:calib-R}
\end{align}
\end{subequations}Using (\ref{eq:D-control}) and (\ref{eq:calib-D})
in (\ref{eq:H-before-D}) gives 
\begin{equation}
\mathsf{H}\le\frac{\tau_{\delta,t}}{1-\rho_{\delta,t}}\frac{1+3\rho_{\delta,t}}{1+\rho_{\delta,t}}\normtwo{A\sigma}.\label{eq:H-from-residual}
\end{equation}
Combine (\ref{eq:residual-control}), (\ref{eq:calib-R}), and (\ref{eq:H-from-residual}).
 If $\normtwo{A\sigma}>0$, division by $\normtwo{A\sigma}$ yields
\begin{align*}
\normtwo{A\sigma} & \le\frac{\tau_{\delta,t}}{1-\rho_{\delta,t}}\frac{(1+3\rho_{\delta,t})(3+\rho_{\delta,t})}{(1+\rho_{\delta,t})^{2}}\sqrt{\rr}\,\lambda\\
 & \le\frac{4\tau_{\delta,t}}{1-\rho_{\delta,t}}\sqrt{\rr}\,\lambda,
\end{align*}
 where 
\[
\frac{(1+3\rho)(3+\rho)}{(1+\rho)^{2}}\le4,\qquad0\le\rho<1.
\]
This proves (\ref{eq:abstract-residual}).  If $\normtwo{A\sigma}=0$,
then (\ref{eq:H-from-residual}) gives $\mathsf{H}=0$, and Lemma~\ref{lem:head-controls}
gives $\sigma=0$. Also $(1+3\rho)/(1+\rho)\le2$, so (\ref{eq:H-from-residual})
and (\ref{eq:abstract-residual}) give 
\[
\mathsf{H}\le\frac{8\tau_{\delta,t}^{2}}{(1-\rho_{\delta,t})^{2}}\sqrt{\rr}\,\lambda.
\]
The calibration makes the square-root factor in (\ref{eq:error-control})
at most $2$.  Equation~(\ref{eq:error-control}) therefore proves
(\ref{eq:abstract-error}).
\end{IEEEproof}

\subsection{Verification for the vector LASSO}

We now verify~(\ref{eq:basic}) for the vector estimator with the
dual-noise level $\lambda_{0}=\norminf{A^{T}\xi}$.
\begin{lem}[Vector $\ell_{1}$ perturbation]
\label{lem:vector-perturbation} Let $x_{\star}$ be $\rr$-sparse.
 For every $\sigma\in\R^{d}$, 
\begin{equation}
\normone{x_{\star}+\sigma}-\normone{x_{\star}}\ge\normone{\sigma_{\tail}}-\normone{\sigma_{\head}}.\label{eq:vector-perturbation}
\end{equation}
\end{lem}
\begin{IEEEproof}
Let $S=\supp(x_{\star})$, padded to size $\rr$.  Then 
\[
\normone{x_{\star}+\sigma}\ge\normone{x_{\star}}-\normone{\sigma_{S}}+\normone{\sigma_{S^{c}}}.
\]
 The best-$\rr$ ordering gives $\normone{\sigma_{\head}}\ge\normone{\sigma_{S}}$
and $\normone{\sigma_{\tail}}\le\normone{\sigma_{S^{c}}}$.
\end{IEEEproof}
\begin{lem}[Vector global-optimality inequality]
\label{lem:vector-basic} Under the assumptions of Theorem~\ref{thm:vector-main},
the error $\sigma=x_{\lambda,r}-x_{\star}$ satisfies (\ref{eq:basic})
with $\lambda_{0}=\norminf{A^{T}\xi}.$
\end{lem}
\begin{IEEEproof}
Since $x_{\star}$ is feasible for (\ref{eq:vector-lasso}), global
optimality gives 
\[
\frac{1}{2}\normtwo{A\sigma-\xi}^{2}+\lambda\normone{x_{\star}+\sigma}\le\frac{1}{2}\normtwo{\xi}^{2}+\lambda\normone{x_{\star}}.
\]
 Expanding the square and applying Lemma~\ref{lem:vector-perturbation}
gives 
\[
\frac{1}{2}\normtwo{A\sigma}^{2}+\lambda\bigl(\normone{\sigma_{\tail}}-\normone{\sigma_{\head}}\bigr)\le\ip{\xi}{A\sigma}\le\norminf{A^{T}\xi}\normone{\sigma},
\]
which is exactly (\ref{eq:basic}).
\end{IEEEproof}
\begin{IEEEproof}[Proof of \thmref{vector-main}]
The calibration and $\lambda>0$ imply $\lambda>\lambda_{0}:=\norminf{A^{T}\xi}$.
All of the assumptions in \lemref{head-controls} are satisfied, so
we apply \lemref{closure} to $\sigma=x_{\lambda,r}-x_{\star}$ to
complete the proof.
\end{IEEEproof}

\subsection{Verification for matrices}

The matrix proof reduces the preceding argument to the singular values
of the error. Let $d=\min\{d_{1},d_{2}\}$ and write
\begin{equation}
E:=X_{\lambda,r}-X_{\star}=U\diag(\sigma)V^{T}\label{eq:Esvd}
\end{equation}
where $\sigma=(\sigma_{1},\ldots,\sigma_{d})^{T}$. Restricting $\A$
to the fixed singular directions $U,V$ defines the induced vector
sensing matrix 
\begin{equation}
Az:=\A\!\left(U\diag(z)V^{T}\right)\text{ for all }z\in\R^{d},\label{eq:induced-A}
\end{equation}
so that $A\sigma=\A(E)$, $\normtwo{\sigma}=\normF E$, and $\normone{\sigma}=\normnuc E$.
The next three lemmas show that the nuclear-norm perturbation, global-optimality
inequality, and RIP hypothesis reduce exactly to their vector counterparts.
\begin{lem}[Nuclear-norm perturbation]
\label{lem:matrix-perturbation} If $\rank(X_{\star})\le\rr$, then
\begin{equation}
\normnuc{X_{\star}+E}-\normnuc{X_{\star}}\ge\normone{\sigma_{\tail}}-\normone{\sigma_{\head}}.\label{eq:matrix-perturbation}
\end{equation}
\end{lem}
\begin{IEEEproof}
Mirsky's singular-value inequality~\cite{Mirsky1960} gives 
\[
\normnuc{X_{\star}+E}\ge\sum_{j}|\sigma_{j}(X_{\star})-\sigma_{j}(E)|.
\]
Since $|a|\ge a$ and $|a|\ge-a$, we use different lower bounds on
the two sides of $r_{\star}$:
\[
|\sigma_{j}(X_{\star})-\sigma_{j}(E)|\ge\begin{cases}
\sigma_{j}(X_{\star})-\sigma_{j}(E) & \text{for }j\le r_{\star},\\
\sigma_{j}(E)-\sigma_{j}(X_{\star}) & \text{for }j>r_{\star}.
\end{cases}
\]
Summing gives 
\begin{align*}
\normnuc{X_{\star}+E}\ge & \sum_{j\le r_{\star}}\sigma_{j}(X_{\star})-\sum_{j\le r_{\star}}\sigma_{j}(E)\\
 & +\sum_{j>r_{\star}}\sigma_{j}(E)-\sum_{j>r_{\star}}\sigma_{j}(X_{\star}).
\end{align*}
Finally, $\sum_{j\le r_{\star}}\sigma_{j}(X_{\star})=\normnuc{X_{\star}}$
and $\sigma_{j}(X_{\star})=0$ for $j>r_{\star}$. 
\end{IEEEproof}
\begin{lem}[Matrix global-optimality inequality]
\label{lem:matrix-basic} Under the assumptions of \thmref{matrix-main},
the singular-value vector $\sigma$ satisfies (\ref{eq:basic}) with
the induced matrix $A$ in (\ref{eq:induced-A}) and $\lambda_{0}=\normop{\A^{*}(\xi)}.$
\end{lem}
\begin{IEEEproof}
Global optimality against the feasible point $X_{\star}$ gives 
\[
\frac{1}{2}\normtwo{\A(E)-\xi}^{2}+\lambda\normnuc{X_{\star}+E}\le\frac{1}{2}\normtwo{\xi}^{2}+\lambda\normnuc{X_{\star}}.
\]
 Expand the square, apply \lemref{matrix-perturbation}, and use 
\[
\ip{\xi}{\A(E)}=\ip{\A^{*}(\xi)}E\le\normop{\A^{*}(\xi)}\normnuc E=\lambda_{0}\normone{\sigma}.
\]
\end{IEEEproof}
\begin{lem}[Matrix RIP implies vector RIP]
\label{lem:matrix-to-vector} Let $U\in\R^{d_{1}\times d},V\in\R^{d_{2}\times d}$
have orthonormal columns. If $\A\in\RIPm(\delta,k)$, then the matrix
$A\in\R^{m\times d}$ satisfying $Az=\A(U\diag(z)V^{T})$ for all
$z\in\R^{d}$ also satisfies $A\in\RIPv(\delta,k)$.
\end{lem}
\begin{IEEEproof}
For any $z\in\R^{d}$, we have $\rank(U\diag(z)V^{T})=\normzero z$
and $\normF{U\diag(z)V^{T}}=\normtwo z$. Hence, for any $\normzero z\le k$,
we have $\normtwo{Az}=\normtwo{\A(U\diag(z)V^{T})}\ge\sqrt{1-\delta}\normF{U\diag(z)V^{T}}=\sqrt{1-\delta}\normtwo z$.
The upper-bound follows identically. 
\end{IEEEproof}
\begin{IEEEproof}[Proof of Theorem~\ref{thm:matrix-main}]
The calibration and $\lambda>0$ imply $\lambda>\lambda_{0}:=\normop{\A^{*}(\xi)}$.
 Apply \lemref{closure} to the singular-value vector $\sigma$ and
the $A$ in (\ref{eq:induced-A}). \lemref{matrix-to-vector} verifies
RIP, \lemref{matrix-basic} verifies the estimator interface, and
$A\sigma=\A(E)$ and $\normtwo{\sigma}=\normF E$ converts the conclusions
back to matrices.
\end{IEEEproof}

\section{Necessity by Counterexamples}

We now show, for a fixed rank ratio $t\coloneqq k/\rr$, that the
RIP recovery threshold $\delta<\sharpdelta(t)$ cannot be uniformly
improved, due to existence of counterexamples at every RIP constant
$\sharpdelta(t)+\varepsilon$ for $\varepsilon>0$. 
\begin{thm}
\label{thm:matrix-sharpness}Fix integers $r_{\star}\ge1$ and $k\ge2$.
Set $t\coloneqq k/\rr$ and
\begin{equation}
\cC(t)\coloneqq1+\frac{(3t-4)_{+}}{2-t+\sqrt{t(t-1)_{+}}}.\label{eq:Cdef}
\end{equation}
If $\delta\ge\sharpdelta(t)+1/\rr$, then:
\begin{itemize}
\item (Matrix) There exist adversarial data $(X_{\star},\A,b)$ satisfying
\[
\rank(X_{\star})=\rr,\quad\A\in\RIPm(\delta,k),\quad b=\A(X_{\star}),
\]
such that every global minimizer $X_{\lambda,r}$ in (\ref{eq:matrix-lasso})
satisfies 
\begin{equation}
\lambda>0,\;r\ge\cC(t)\rr\implies\normF{X_{\lambda,r}-X_{\star}}\ge\normF{X_{\star}}.\label{eq:matrix-failure}
\end{equation}
\item (Vector) There exist adversarial data $(x_{\star},A,b)$ satisfying
\[
\normzero{x_{\star}}=\rr,\quad A\in\RIPv(\delta,k),\quad b=Ax_{\star},
\]
such that every global minimizer $x_{\lambda,r}$ in (\ref{eq:vector-lasso})
satisfies 
\begin{equation}
\lambda>0,\;r\ge\cC(t)\rr\implies\normtwo{x_{\lambda,r}-x_{\star}}\ge\normtwo{x_{\star}}.\label{eq:vector-failure}
\end{equation}
\end{itemize}
\end{thm}

We first record KKT and uniqueness conditions for the convex matrix
LASSO. The nonstrict condition (\ref{eq:kkt}) for global optimality
is standard; see e.g. \cite[Prop.~3]{bach2008consistency}. However,
the strict KKT condition $\normop W<1$ for uniqueness does not always
hold in our examples, so we adopt a simple alternative. 
\begin{lem}
\label{lem:convex-opt}For $X\in\R^{d_{1}\times d_{2}}$, let $X=U\Sigma V^{T}$
denote its compact SVD. Suppose there exists $W$ such that 
\begin{gather}
\A^{*}[\A(X)-b]+\lambda(UV^{T}+W)=0,\nonumber \\
U^{T}W=0,\quad WV=0,\quad\normop W\le1.\label{eq:kkt}
\end{gather}
Then $X$ is a global minimizer of (\ref{eq:matrix-lasso}) for every
$r\ge\rank(X)$. If, in addition, 
\begin{equation}
H\ne0,\A(H)=0\implies\normnuc{X+H}>\normnuc X\label{eq:fiber}
\end{equation}
then $X$ is the \emph{unique} global minimizer for every $r\ge\rank(X)$. 
\end{lem}
\begin{IEEEproof}
Set $\phi(X)\coloneqq\frac{1}{2}\norm{\A(X)-b}^{2}+\lambda\normnuc X$.
Condition (\ref{eq:kkt}) is exactly $0\in\partial\phi(X)$, so $X$
minimizes the unrestricted convex LASSO and hence every rank-restricted
problem with $r\ge\rank(X)$. For uniqueness, let $Y$ be another
convex minimizer. If $\A(Y)\ne\A(X)$, strict convexity of $z\mapsto\frac{1}{2}\|z-b\|_{2}^{2}$
and convexity of the nuclear norm give $\phi(\frac{X+Y}{2})<\frac{\phi(X)}{2}+\frac{\phi(Y)}{2}$,
a contradiction. Therefore $\A(Y)=\A(X)$, or equivalently, $Y=X+H$
for some $H\in\ker\A$. Condition (\ref{eq:fiber}) forces $H=0$,
proving uniqueness.
\end{IEEEproof}
We next state our two-block adversarial construction. The ground truth
$X_{\star}$ and bad minimizer $X_{\lambda,r}$ are placed in orthogonal
subspaces. The sensing map annihilates the secant $E=X_{\lambda,r}-X_{\star}$
joining them, while the orthogonal directions supply the subgradients
to establish optimality at $X_{\lambda,t}$. Here, $q$ is the rank
of the bad minimizer, and $\alpha$ is a block-amplitude parameter,
which is carefully tuned in \lemref{param-choice} to achieve $\A\in\RIPm(\delta,k)$
for the desired RIP constant $\delta$.
\begin{lem}[Optimality]
\label{lem:lasso-path}Given integer $q\ge\rr$, and reals $0<\alpha<\rr/q,$
and $\delta>0$, let $\Pi_{1},\Pi_{2}\in\R^{d\times d}$ be orthogonal
projectors satisfying 
\[
\Pi_{1}\Pi_{2}=0,\qquad\rank(\Pi_{1})=\rr,\qquad\rank(\Pi_{2})=q,
\]
and set $G\coloneqq\Pi_{1}-\alpha\Pi_{2}$. For the ground truth $X_{\star}\coloneqq\Pi_{1}$,
define $\A$ by 
\begin{equation}
\A^{*}\A(Z)\coloneqq(1+\delta)\left(Z-\frac{\ip GZ}{\normF G^{2}}G\right)\quad\text{for all }Z\in\R^{d\times d},\label{eq:Adef}
\end{equation}
and let $b=\A(X_{\star})$. Then every global minimizer $X_{\lambda,r}$
of (\ref{eq:matrix-lasso}) satisfies 
\begin{equation}
\lambda>0,\quad r\ge q\implies\normF{X_{\lambda,r}-X_{\star}}\ge\normF{X_{\star}}.\label{eq:bad-path-failure}
\end{equation}
\end{lem}
\begin{IEEEproof}
We verify that $X_{\lambda}$ below satisfy all conditions in \lemref{convex-opt}:
\[
\lambda_{1}\coloneqq(1+\delta)\frac{\alpha\rr}{\normF G^{2}},\quad X_{\lambda}\coloneqq\alpha\left(1-\frac{\lambda}{\lambda_{1}}\right)_{+}\Pi_{2}.
\]
Since $\A(G)=\A(X_{\star}-\alpha\Pi_{2})=0$ by definition, we have
\begin{gather*}
\A^{*}\A(X_{\star})=\A^{*}\A(\alpha\Pi_{2})=\lambda_{1}\left(\Pi_{2}+\frac{q\alpha}{\rr}\Pi_{1}\right)\\
\A^{*}\A(X_{\lambda}-X_{\star})=\left[\left(\lambda_{1}-\lambda\right)_{+}-\lambda_{1}\right]\left(\Pi_{2}+\frac{q\alpha}{\rr}\Pi_{1}\right).
\end{gather*}
For $0<\lambda<\lambda_{1}$, (\ref{eq:kkt}) holds with $UV^{T}=\Pi_{2}$
and $W=(q\alpha/\rr)\Pi_{1}$. For $\lambda\ge\lambda_{1}$, (\ref{eq:kkt})
holds with $W=\frac{\lambda_{1}}{\lambda}\left(\Pi_{2}+\frac{q\alpha}{\rr}\Pi_{1}\right)$.
In both cases, $\normop W\le1$ because $q\alpha<\rr$. Next, (\ref{eq:Adef})
gives $\ker A=\{uG:u\in\R\}$, and
\begin{align*}
\normnuc{X_{\lambda}+uG} & =\rr|u|+q\alpha\left|\left(1-{\textstyle \frac{\lambda}{\lambda_{1}}}\right)_{+}-u\right|.
\end{align*}
Since $\rr>q\alpha$, this expression is uniquely minimized at $u=0$.
Therefore, (\ref{eq:fiber}) holds. It follows from $X_{\lambda,r}X_{\star}=0$
that $\|X_{\lambda,r}-X_{\star}\|_{\F}\ge\|X_{\star}\|_{\F}$.
\end{IEEEproof}
A key feature of the identity-minus-rank-one construction in (\ref{eq:Adef})
is that its RIP constant can be computed exactly at every rank $k$.
This same insight was previously utilized in~\cite{CaiZhang2014,McRaeZhang2026}.

\begin{lem}
\label{lem:rip}The map $\A$ in (\ref{eq:Adef}) satisfies $\A\in\RIPm(\delta,k)$
if and only if 
\begin{equation}
\frac{\min\{k,\rr\}+\alpha^{2}\min\{(k-\rr)_{+},q\}}{\rr+\alpha^{2}q}\le\frac{2\delta}{1+\delta}.\label{eq:rip_apqk}
\end{equation}
\end{lem}
\begin{IEEEproof}
Set $\hat{G}\coloneqq G/\normF G$. Then, $\|\A(E)\|^{2}/\normF E^{2}\le1+\delta$,
while the Eckart--Young--Mirsky theorem yields
\begin{gather*}
\min_{\begin{subarray}{c}
\rank(E)\le k\\
\normF E=1
\end{subarray}}\frac{\|\A(E)\|^{2}}{1+\delta}=\min_{\begin{subarray}{c}
\rank(E)\le k\\
\normF E=1
\end{subarray}}\left\{ \normF E^{2}-\ip{\hat{G}}E^{2}\right\} \\
=1-\max_{\begin{subarray}{c}
\rank(E)\le k\\
\normF E=1
\end{subarray}}\left\{ \ip{\hat{G}}E^{2}\right\} =1-\sum_{i\le k}\sigma_{i}^{2}(\hat{G}).
\end{gather*}
Therefore, $\A\in\RIPm(\delta,k)$ if and only if $\sum_{i\le k}\sigma_{i}^{2}(\hat{G})\le\frac{2\delta}{1+\delta}$.
Computing the singular values of $\hat{G}$ yields (\ref{eq:rip_apqk}).
\end{IEEEproof}
It remains to select the bad minimizer rank $q$ and tune the parameter
$0<\alpha<\rr/q$ to establish $\A\in\RIPm(\delta,k)$ for $\delta>\sharpdelta(t)$.
This becomes surprisingly tricky for the higher-order branch, due
to the need for $q$ to be an integer. The headroom $\delta-\sharpdelta(t)\ge\frac{1}{\rr}$
is needed to absorb the rounding error, written below as $\varepsilon$.
\begin{lem}[RIP]
\label{lem:param-choice}Fix integers $r_{\star}\ge1$ and $k\ge2$,
set $t\coloneqq k/\rr$. For $\delta\ge\sharpdelta(t)+1/\rr$, choose
\[
q\coloneqq\left\lfloor \cC(t)\rr\right\rfloor ,\;\alpha\coloneqq\frac{1}{\cC(t)}\sqrt{\frac{\sharpdelta(t)}{\delta-\varepsilon}},\;\varepsilon\coloneqq\cC(t)-\frac{\left\lfloor \cC(t)\rr\right\rfloor }{\rr},
\]
where $\cC(t)$ is defined in (\ref{eq:Cdef}). Then, $0<\alpha<\rr/q$
holds, and the map $\A$ in (\ref{eq:Adef}) satisfies $\A\in\RIPm(\delta,k)$. 
\end{lem}
\begin{IEEEproof}
Write $\sharpdelta\equiv\sharpdelta(t)$ and $\cC\equiv\cC(t)$. Suppose
first $0<t\le4/3$. Then $\sharpdelta=t/(4-t)$ and $\cC=1$, so $q=\rr$,
$\varepsilon=0$, and $\alpha^{2}=\sharpdelta/\delta<1$. \lemref{rip}
reduces the required RIP condition to 
\begin{equation}
\frac{2\delta}{1+\delta}\ge\begin{cases}
{\displaystyle \frac{t}{1+\alpha^{2}},} & 0<t\le1,\\[1.2ex]
{\displaystyle \frac{1+(t-1)\alpha^{2}}{1+\alpha^{2}},} & 1\le t\le4/3.
\end{cases}\label{eq:low-visible}
\end{equation}
Subtracting the two sides and substituting $\alpha^{2}=\sharpdelta(t)/\delta<1$
confirms that (\ref{eq:low-visible}) holds, and therefore $\A\in\RIPm(\delta,k)$.

Now suppose $t>4/3$. In this regime 
\begin{gather*}
\sharpdelta=\sqrt{\frac{t-1}{t}},\qquad\cC=\frac{\sharpdelta(t)}{1-\sharpdelta(t)}>1.
\end{gather*}
Because $0\le\varepsilon<1/\rr$ and $\delta\ge\sharpdelta+1/\rr$,
we have $\delta-\varepsilon>\sharpdelta$. Therefore $0<\alpha<\cC^{-1}<(\cC-\varepsilon)^{-1}=\rr/q$.
\lemref{rip} reduces the RIP condition to 
\[
\frac{1+(t-1)\alpha^{2}}{1+(\cC-\varepsilon)\alpha^{2}}\le\frac{2\delta}{1+\delta}.
\]
Equivalently, it suffices to validate the following 
\[
\psi(\cC-\varepsilon)\le\delta\quad\text{ where }\psi(\tau)\coloneqq\frac{1+(t-1)\alpha^{2}}{1+(2\tau-t+1)\alpha^{2}}.
\]
Substitution of $\alpha$ and $\cC=\sharpdelta/(1-\sharpdelta)$ gives
\begin{multline*}
\delta-\varepsilon-\psi(\cC)=\\
\frac{(\delta-\varepsilon-\sharpdelta)\bigl[(\delta-\varepsilon)(1+\sharpdelta)+1-\sharpdelta\bigr]}{(\delta-\varepsilon-\sharpdelta)(1+\sharpdelta)+2}>0.
\end{multline*}
For $\tau\in[\cC-\varepsilon,\cC]$, we have $\tau\ge1$, $\tau>t-1$,
and $\alpha<1$, so 
\[
0<-\psi'(\tau)=\frac{2\alpha^{2}[1+(t-1)\alpha^{2}]}{[1+(2\tau-t+1)\alpha^{2}]^{2}}<1.
\]
Therefore $\psi(\cC-\varepsilon)\le\psi(\cC)+\varepsilon<\delta,$
which proves $\A\in\RIPm(\delta,k)$. 
\end{IEEEproof}
Finally, we obtain the vector example by imposing a diagonal restriction
onto the matrix example. The fact that RIP transfers across was already
established in \lemref{matrix-to-vector}.
\begin{lem}[Matrix-to-vector reduction]
\label{lem:diagonal-reduction}Suppose the rank-restricted matrix
LASSO (\ref{eq:matrix-lasso}) with linear map $\A$ and noiseless
measurements $b=\A(X_{\star})$ admits a unique solution $X_{\lambda,r}$
with a common SVD as the ground truth
\[
X_{\star}=U\diag(x_{\star})V^{T},\quad X_{\lambda,r}=U\diag(x_{\lambda,r})V^{T}
\]
where $U\in\R^{d_{1}\times d}$ and $V\in\R^{d_{2}\times d}$ have
orthonormal columns. Define the matrix $A$ by 
\[
Ax=\A\!\left(U\diag(x)V^{T}\right),\qquad x\in\R^{d}.
\]
Then $Ax_{\star}=b$ and $\normtwo{x_{\star}}=\normF{X_{\star}}$,
and the sparsity-restricted vector LASSO (\ref{eq:vector-lasso})
with linear map $A$, and measurements $b$ has unique solution $x_{\lambda,r}$
satisfying $\normtwo{x_{\lambda,r}-x_{\star}}=\normF{X_{\lambda,r}-X_{\star}}$.
\end{lem}
\begin{IEEEproof}
The definition of $A$ gives $Ax_{\star}=\A(X_{\star})=b$, while
orthonormality of $U,V$ gives $\normtwo{x_{\star}}=\normF{X_{\star}}$,
$\normtwo{x_{\lambda,r}-x_{\star}}=\normF{X_{\lambda,r}-X_{\star}}$,
$\rank\!\left(U\diag(x)V^{T}\right)=\|x\|_{0},$ and $\left\Vert U\diag(x)V^{T}\right\Vert _{\nuc}=\|x\|_{1}$.
To prove that $x_{\lambda,r}$ is uniquely optimal, observe that the
vector LASSO is equivalent to the restriction of the matrix LASSO
with the added constraint $X=U\diag(x)V^{T}$ for $x\in\R^{d}$. Since
$X_{\lambda,r}=U\diag(x_{\lambda,r})V^{T}$ belongs to this restriction
and uniquely minimizes the relaxation, it uniquely minimizes the restriction.
Equivalently, $x_{\lambda,r}$ uniquely minimizes the vector problem. 
\end{IEEEproof}
\begin{IEEEproof}[Proof of Theorem~\ref{thm:matrix-sharpness}]
Choose $q$ and $\alpha$ from \lemref{param-choice}, and take $d>q+\rr$.
\lemref{rip} verifies the required matrix RIP, whereas \lemref{lasso-path}
gives a unique bad minimizer of $\rank(X_{\lambda})\le q=\lfloor\cC(t)\rr\rfloor\le\cC(t)\rr$.
Every search rank $r\ge\cC(t)\rr$ contains this minimizer, and the
matrix recovery failure (\ref{eq:matrix-failure}) follows from (\ref{eq:bad-path-failure}).
For the vector instance, \lemref{matrix-to-vector} confirms the required
vector RIP, whereas \lemref{diagonal-reduction} extends the same
failure to the vector case (\ref{eq:vector-failure}). 
\end{IEEEproof}

\section{Proof of \propref{HTR}\protect\label{sec:proof-proposition}}

We now prove \propref{HTR}, the central geometric estimate underlying
our sufficient condition. Throughout this section, fix $A\in\R^{m\times d}$
with $A\in\RIPv(\delta,k)$ and an arbitrary $\sigma\in\R^{d}$. Simultaneously
permuting the coordinates of $\sigma$ and the columns of $A$, we
may assume $|\sigma_{1}|\ge\cdots\ge|\sigma_{d}|$.

The basic strategy dates to Candès and Tao~\cite{CandesTao2007}.
Split the error into a sparse, spiky head and a dense, diffuse tail,
\[
\sigma=\sigma_{\head}+\sigma_{\tail},\qquad\sigma_{\head}=(\sigma_{1},\ldots,\sigma_{\rr},0,\ldots,0),
\]
and then represent the tail by sparse pieces
\[
\begin{gathered}\sigma_{\tail}=\sum_{i}\sigma_{\tail,i},\qquad\normzero{\sigma_{\tail,i}}\le q.\end{gathered}
\]
RIP preserves the norms of sparse pieces $\norm{A\sigma_{\head}}\approx\norm{\sigma_{\head}}$
and controls the measured inner products of disjoint pieces $\ip{A\sigma_{\tail,i}}{A\sigma_{\tail,j}}\approx\ip{\sigma_{\tail,i}}{\sigma_{\tail,j}}=0$,
so that the whole vector could be preserved 
\begin{align*}
\normtwo{A\sigma}^{2} & =\normtwo{A(\sigma_{\head}+{\textstyle \sum_{i}}\sigma_{\tail,i})}^{2}\\
 & \approx\normtwo{\sigma_{\head}}^{2}+{\textstyle \sum_{i}}\normtwo{\sigma_{\tail,i}}^{2}=\normtwo{\sigma}^{2}.
\end{align*}
The proof must choose the pieces so that the accumulated RIP errors
do not destroy this comparison. When $k$ is large enough, the largest
tail coordinates can first be attached to the head, leaving only a
diffuse remainder to be decomposed.

A rigid partition of the tail into consecutive disjoint blocks is
simple but conservative, because its first few blocks remain too spiky.
The critical insight behind Cai and Zhang's proof of the sharp threshold
$\delta<\sharpdelta(k/\rr)$ for constrained recovery~\cite{CaiZhang2014}
is to replace this partition by a convex decomposition of the diffuse
tail into sparse atoms. Equivalently, 
\[
\sigma_{\tail}=\sum_{i}p_{i}\sigma_{\tail,i}=\E[\bm{\sigma}_{\tail}],\qquad p_{i}\ge0,\quad\sum_{i}p_{i}=1,
\]
where every realization $\sigma_{\tail,i}$ is sparse. We refer to
this operation as \emph{atomization}. The following lemma is the expectation-form
restatement of the Cai--Zhang sparse-polytope decomposition.
\begin{lem}[{Sparse-polytope atomization~\cite[Lem~1.1]{CaiZhang2014}}]
\label{lem:polytope} Let $x\in\R^{d}$, let $a\ge0$, and let $q\ge1$
be an integer.  If $\norminf x\le a$ and $\normone x\le qa$, then
there is a finitely supported random vector $\bm{z}$ such that $x=\E[\bm{z}]$,
$\E\normtwo{\bm{z}}^{2}\le a\normone x$, and almost surely, $\normzero{\bm{z}}\le q$,
$\supp(\bm{z})\subseteq\supp(x)$. 
\end{lem}
The atoms need not have disjoint supports. Each instead respects the
sparsity budget while retaining the tail's diffuse entrywise scale,
and their controlled second moment avoids the spikiness inherent in
a rigid block partition.

The lower-order branch $2\le k<4\rr/3$ requires a further idea. When
$\rr\le k<4\rr/3$, the high-order atomization argument no longer
reaches the sharp threshold; when $k<\rr$, even the full head is
outside the RIP order. One can convert rank-$k$ RIP into a higher-order
bound, but this inflates the RIP constant too much for a sharp guarantee.

\begin{lem}[{RIP order conversion~\cite[Lem~4.1]{CaiZhang2013}\cite[Lem~5.1]{CaiZhang2014}}]
\label{lem:ric-comparison} If $A$ satisfies $(\delta,k)$ vector
RIP and $k'\ge k\ge2$, then it satisfies the RIP inequalities of
order $k'$ with $\delta':=\left(2k'/k-1\right)\delta.$
\end{lem}
Zhang and Li's breakthrough~\cite{ZhangLi2018} is to randomize the
head as well. They uniformly sample an $a$-sparse head piece and
pair it with a $b$-sparse tail atom, where $a+b=k$, and symmetrically
pair a $b$-sparse head piece with an $a$-sparse tail atom. Averaging
these paired RIP inequalities produces the cancellations needed to
recover the full sharp threshold $\delta<\sharpdelta(k/\rr)$ for
constrained recovery across all $k\ge2$ and $\rr\ge1$. 

Our proof of \propref{HTR} uses both mechanisms. In fact, a proof
of the easier high-order branch $k\ge4\rr/3$ can already be found
in Wang, Zhang, and Wang \cite[Lem.~2]{WangZhangWang2021}, but our
extension to the lower-order branch $2\le k<4\rr/3$ is new. We prove
both branches in full, both for completeness, and also because the
high-order argument provides a useful tutorial for the substantially
more difficult low-order branch. 

\subsection{High-order branch $t\ge4/3$}

Here $k>\rr$, so the RIP budget can accommodate the head together
with part of the tail. We retain every tail coordinate above a common
threshold in a deterministic \emph{big} component and atomize only
the diffuse \emph{small} remainder. This makes every randomized atom
$k$-sparse while preserving the mean of the complete error. Set 
\begin{gather}
\gamma:=\frac{\normone{\sigma_{\tail}}}{k-\rr},\qquad\Lambda:=\{j>\rr:|\sigma_{j}|>\gamma\},\label{eq:high-sets}\\
\sigma_{\mathrm{big}}:=\sigma_{[\rr]\cup\Lambda},\qquad\sigma_{\mathrm{small}}:=\sigma-\sigma_{\mathrm{big}}.\label{eq:high-big-small}
\end{gather}
We atomize $\sigma_{\mathrm{small}}$ and obtain the following moments. 
\begin{lem}[Big--small atomization]
\label{lem:high-randomization} There is a finitely supported random
vector $\bm{u}$ such that 
\begin{equation}
\E[\bm{u}]=\sigma_{\mathrm{small}},\qquad\E\normtwo{\bm{u}}^{2}\le\frac{\normone{\sigma_{\tail}}^{2}}{k-\rr},\label{eq:high-randomization}
\end{equation}
 and, almost surely, 
\begin{equation}
\supp(\bm{u})\subseteq\supp(\sigma_{\mathrm{small}}),\qquad\normzero{\sigma_{\mathrm{big}}}+\normzero{\bm{u}}\le k.\label{eq:high-support}
\end{equation}
\end{lem}
\begin{IEEEproof}
If $\sigma_{\tail}=0$, then take $\bm{u}=0$. Otherwise, if $\sigma_{\tail}\ne0$,
then every coordinate in $\Lambda$ has magnitude larger than $\gamma$,
so 
\[
|\Lambda|\gamma<\normone{\sigma_{\Lambda}}\le\normone{\sigma_{\tail}}=(k-\rr)\gamma,
\]
and hence $|\Lambda|<k-\rr$. By construction, $\norminf{\sigma_{\mathrm{small}}}\le\gamma,$
while 
\[
\normone{\sigma_{\mathrm{small}}}=\normone{\sigma_{\tail}}-\normone{\sigma_{\Lambda}}\le(k-\rr-|\Lambda|)\gamma.
\]
 Since $\sigma_{\mathrm{big}}$ and $\bm{u}$ have disjoint supports,
\[
\normzero{\sigma_{\mathrm{big}}}+\normzero{\bm{u}}\le(\rr+|\Lambda|)+(k-\rr-|\Lambda|)=k
\]
 almost surely.
\end{IEEEproof}
Let $\bm{u}$ be the random vector from \lemref{high-randomization},
and define 
\begin{align}
\mathsf{X}:=\E\Big[ & \normtwo{A((1+\delta)\sigma_{\mathrm{big}}+\delta\bm{u})}^{2}\nonumber \\[-1mm]
 & -\normtwo{A((1-\delta)\sigma_{\mathrm{big}}-\delta\bm{u})}^{2}\Big].\label{eq:high-X}
\end{align}
By (\ref{eq:high-support}), both vectors inside (\ref{eq:high-X})
are $k$-sparse almost surely and are therefore preserved by $A\in\RIPv(\delta,k)$. 
\begin{lem}[High-order RIP comparison]
\label{lem:high-paired} The quantity (\ref{eq:high-X}) satisfies
\begin{subequations}\label{eq:high-paired-bounds}
\begin{align}
\mathsf{X} & \le4\delta\sqrt{1+\delta}\,\normtwo{\sigma_{\mathrm{big}}}\normtwo{A\sigma},\label{eq:high-upper}\\
\mathsf{X} & \ge2\delta(1-\delta^{2})\normtwo{\sigma_{\mathrm{big}}}^{2}-\frac{2\delta^{3}}{k-\rr}\normone{\sigma_{\tail}}^{2}.\label{eq:high-lower}
\end{align}
\end{subequations}
\end{lem}
\begin{IEEEproof}
Expanding (\ref{eq:high-X}) and using $\E[\bm{u}]=\sigma_{\mathrm{small}}$
gives 
\begin{align*}
\mathsf{X} & =4\delta\ip{A\sigma_{\mathrm{big}}}{A(\sigma_{\mathrm{big}}+\E[\bm{u}])}\\
 & =4\delta\ip{A\sigma_{\mathrm{big}}}{A\sigma}.
\end{align*}
Cauchy--Schwarz and upper RIP prove (\ref{eq:high-upper}). For the
lower bound, apply lower RIP to $(1+\delta)\sigma_{\mathrm{big}}+\delta\bm{u}$
and upper RIP to $(1-\delta)\sigma_{\mathrm{big}}-\delta\bm{u}$.
 Their supports are disjoint, so the resulting Euclidean expression
is 
\begin{align*}
 & (1-\delta)\left[(1+\delta)^{2}\normtwo{\sigma_{\mathrm{big}}}^{2}+\delta^{2}\normtwo{\bm{u}}^{2}\right]\\
 & \quad-(1+\delta)\left[(1-\delta)^{2}\normtwo{\sigma_{\mathrm{big}}}^{2}+\delta^{2}\normtwo{\bm{u}}^{2}\right]\\
 & =2\delta(1-\delta^{2})\normtwo{\sigma_{\mathrm{big}}}^{2}-2\delta^{3}\normtwo{\bm{u}}^{2}.
\end{align*}
Take expectations and using (\ref{eq:high-randomization}) proves
(\ref{eq:high-lower}).
\end{IEEEproof}
\begin{IEEEproof}[Proof of Proposition~\ref{prop:HTR} for $t\ge4/3$]
First suppose $\delta>0$. Comparing (\ref{eq:high-upper}) and (\ref{eq:high-lower})
and dividing by $2\delta$ gives 
\begin{multline}
(1-\delta^{2})\normtwo{\sigma_{\mathrm{big}}}^{2}-2\sqrt{1+\delta}\,\normtwo{A\sigma}\normtwo{\sigma_{\mathrm{big}}}\\[-1mm]
-\frac{\delta^{2}}{k-\rr}\normone{\sigma_{\tail}}^{2}\le0.\label{eq:high-quadratic}
\end{multline}
Solving this quadratic for the nonnegative quantity $\normtwo{\sigma_{\mathrm{big}}}$
and using $\sqrt{x^{2}+y^{2}}\le x+y$ for $x,y\ge0$ yields 
\begin{align*}
\normtwo{\sigma_{\mathrm{big}}} & \le\frac{2}{(1-\delta)\sqrt{1+\delta}}\normtwo{A\sigma}\\[-1mm]
 & \quad+\frac{\delta}{\sqrt{(1-\delta^{2})(k-\rr)}}\normone{\sigma_{\tail}}\\
 & =\tau_{\delta,t}\normtwo{A\sigma}+\rho_{\delta,t}\frac{\normone{\sigma_{\tail}}}{\sqrt{\rr}}.
\end{align*}
Because $\sigma_{\head}$ and $\sigma_{\Lambda}$ have disjoint supports,
$\mathsf{H}=\normtwo{\sigma_{\head}}\le\normtwo{\sigma_{\mathrm{big}}}$;
by definition, $\mathsf{T}\ge\normone{\sigma_{\tail}}/\sqrt{\rr}$.
 This proves (\ref{eq:HTR}). If $\delta=0$, then $A$ also satisfies
$k$-sparse RIP with every constant $\varepsilon>0$.  Apply the preceding
argument with $\varepsilon$ in place of $\delta$ and let $\varepsilon\downarrow0$.
\end{IEEEproof}

\subsection{Low-order branch $0<t<4/3$}

This section reproduces the Zhang and Li~\cite{ZhangLi2018} construction,
and then makes the necessary modifications to prove the lower-order
branch of \propref{HTR}. Assume $2\le k<4\rr/3$ and choose $\alpha,\beta$
as follows 
\begin{equation}
r_{\star}\ge\alpha\ge\beta\ge1,\qquad\alpha+\beta=k=t\rr.\label{eq:alpha-beta}
\end{equation}
We uniformly sample $\sigma_{\head}$ into $\alpha$-sparse and $\beta$-sparse
pieces, and atomize $\sigma_{\tail}$ into $\beta$-sparse and $\alpha$-sparse
pieces using \lemref{polytope}. We then pair each $\alpha$-sparse
head piece with a $\beta$-sparse tail atom, each $\beta$-sparse
head piece with an $\alpha$-sparse tail atom, propagate their disjoint
$k$-sparse sums through $A$ and evoke $(\delta,k)$-RIP.

\subsubsection{Sampling the head}

Let the head subsets be sampled uniformly according to
\begin{align*}
\bm{T} & \sim\Unif\{T\subseteq[\rr]:|T|=\alpha\},\\[-1mm]
\bm{S} & \sim\Unif\{S\subseteq[\rr]:|S|=\beta\},
\end{align*}
The resulting sparse pieces have explicit first and second moments.
Conditioning two head samples to be disjoint lets their union use
the full budget $k=\alpha+\beta$. 

The following moment calculations are the expectation-normalized form
of Zhang--Li's~\cite[Lem~1]{ZhangLi2018} and \cite[Lem~2]{ZhangLi2018}.
\begin{lem}[Uniform head sampling]
\label{lem:head-randomization} One has \begin{subequations}\label{eq:uniform-head-moments}
\begin{align}
\E[\sigma_{\bm{T}}] & =\frac{\alpha}{\rr}\sigma_{\head},\label{eq:uniform-head-mean}\\
\E\normtwo{\sigma_{\bm{T}}}^{2} & =\frac{\alpha}{\rr}\mathsf{H}^{2},\label{eq:uniform-head-energy}
\end{align}
\end{subequations} and 
\begin{align}
\E\normtwo{A\sigma_{\bm{T}}}^{2}={} & \frac{\alpha(\alpha-1)}{\rr(\rr-1)}\normtwo{A\sigma_{\head}}^{2}\nonumber \\[-1mm]
 & +\frac{\alpha(\rr-\alpha)}{\rr(\rr-1)}\sum_{j\le\rr}\sigma_{j}^{2}\normtwo{A\e_{j}}^{2}.\label{eq:uniform-second}
\end{align}
Moreover, $\normzero{\sigma_{\bm{T}}}\le\alpha$ almost surely. The
analogous identities hold with $(\bm{T},\alpha)$ replaced by $(\bm{S},\beta)$.
\end{lem}
\begin{IEEEproof}
For the uniform $\alpha$-subset, 
\[
\Pr(j\in\bm{T})=\frac{\alpha}{\rr},\qquad\Pr(i,j\in\bm{T})=\frac{\alpha(\alpha-1)}{\rr(\rr-1)}\quad(i\ne j).
\]
 The first probability proves (\ref{eq:uniform-head-mean})--(\ref{eq:uniform-head-energy}).
 Set $g_{j}:=A(\sigma_{j}\e_{j})$.  Expanding the measured squared
norm gives 
\begin{align*}
\E\normtwo{A\sigma_{\bm{T}}}^{2} & =\frac{\alpha}{\rr}\sum_{j\le\rr}\normtwo{g_{j}}^{2}\\
 & \quad+\frac{\alpha(\alpha-1)}{\rr(\rr-1)}\sum_{\substack{i,j\le\rr\\
i\ne j
}
}\ip{g_{i}}{g_{j}}.
\end{align*}
 Using $\sum_{i\ne j}\ip{g_{i}}{g_{j}}=\normtwo{A\sigma_{\head}}^{2}-\sum_{j}\normtwo{g_{j}}^{2}$
proves (\ref{eq:uniform-second}).
\end{IEEEproof}
\begin{lem}[Disjoint head sampling]
\label{lem:disjoint-head} Suppose $0<t\le1$.  Condition $(\bm{T},\bm{S})$
on $\bm{T}\cap\bm{S}=\varnothing$.  The conditional marginals remain
uniform and 
\[
\bm{T}\cup\bm{S}\sim\Unif\{R\subseteq[\rr]:|R|=\alpha+\beta\}.
\]
 Moreover, \begin{subequations}\label{eq:disjoint-head-moments}
\begin{align}
\E\!\left[\normtwo{\sigma_{\bm{T}}+\sigma_{\bm{S}}}^{2}\,\middle|\,\bm{T}\cap\bm{S}=\varnothing\right] & =t\mathsf{H}^{2},\label{eq:disjoint-plus}\\
\E\!\left[\normtwo{\beta\sigma_{\bm{T}}-\alpha\sigma_{\bm{S}}}^{2}\,\middle|\,\bm{T}\cap\bm{S}=\varnothing\right] & =\alpha\beta t\mathsf{H}^{2}.\label{eq:disjoint-minus}
\end{align}
\end{subequations}
\end{lem}
\begin{IEEEproof}
Every admissible ordered disjoint pair has the same unconditional
probability, so conditioning makes the pair uniform over all such
pairs. The marginal and union laws follow by symmetry.  Since the
supports are disjoint, 
\begin{align*}
\normtwo{\sigma_{T}+\sigma_{S}}^{2} & =\sum_{j\in T\cup S}\sigma_{j}^{2},\\
\normtwo{\beta\sigma_{T}-\alpha\sigma_{S}}^{2} & =\beta^{2}\sum_{j\in T}\sigma_{j}^{2}+\alpha^{2}\sum_{j\in S}\sigma_{j}^{2}.
\end{align*}
 Taking the uniform marginals proves (\ref{eq:disjoint-plus})--(\ref{eq:disjoint-minus}).
\end{IEEEproof}

\subsubsection{Atomizing the tail}

The tail is diffuse at the scale measured by $\mathsf{T}$, which
controls both the total mass and the largest coordinate of $\sigma_{\tail}$.
Each required sparsity level can be obtained from the same polytope
lemma as the high-order branch.
\begin{lem}[Polytope tail atomization]
\label{lem:tail-randomization} For every integer $1\le s\le\rr$,
there is a finitely supported random vector $\bm{u}^{(s)}$ such that
\begin{equation}
\E[\bm{u}^{(s)}]=\sigma_{\tail},\qquad\E\normtwo{\bm{u}^{(s)}}^{2}\le\frac{\rr}{s}\mathsf{T}^{2},\label{eq:tail-randomization-moments}
\end{equation}
 with $\normzero{\bm{u}^{(s)}}\le s$ and $\supp(\bm{u}^{(s)})\subseteq\supp(\sigma_{\tail})$
almost surely.
\end{lem}
\begin{IEEEproof}
By \defref{HTD}, we have $\norminf{\sigma_{\tail}}\le\frac{\mathsf{T}}{\sqrt{\rr}}$
and $\normone{\sigma_{\tail}}\le\sqrt{\rr}\,\mathsf{T}$. Since $s\le\rr$,
this further yields $\norminf{\sigma_{\tail}}\le\sqrt{\rr}\,\mathsf{T}/s$.
 Apply Lemma~\ref{lem:polytope} with sparsity $s$ and entry bound
$\sqrt{\rr}\,\mathsf{T}/s$.  Its second-moment estimate gives 
\[
\E\normtwo{\bm{u}^{(s)}}^{2}\le\frac{\sqrt{\rr}\,\mathsf{T}}{s}\normone{\sigma_{\tail}}\le\frac{\rr}{s}\mathsf{T}^{2}.
\]
\end{IEEEproof}
Following Zhang and Li~\cite[Eqns~10-12]{ZhangLi2018}, let $\bm{u}$
and $\bm{v}$ be the randomizations with sparsities $\beta$ and $\alpha$.
It follows from the above that
\begin{equation}
\E\normtwo{\bm{u}}^{2}\le\frac{\rr}{\beta}\mathsf{T}^{2},\qquad\E\normtwo{\bm{v}}^{2}\le\frac{\rr}{\alpha}\mathsf{T}^{2}.\label{eq:uv-moments}
\end{equation}
If additionally $1<t<4/3$, then also let $\bm{w}$ be the randomization
with sparsity $(t-1)\rr=k-\rr$, with
\begin{equation}
\E\normtwo{\bm{w}}^{2}\le\frac{1}{t-1}\mathsf{T}^{2}.\label{eq:w-moment}
\end{equation}
If $t\le1$, we simply set $\bm{w}=0$. 

\subsubsection{RIP comparisons}

For disjointly supported $g,h$ and positive integers $p,q$, define
\begin{equation}
\Psi_{p,q}(g,h):=p^{2}\normtwo{A(g+\tfrac{q}{\rr}h)}^{2}-q^{2}\normtwo{A(g-\tfrac{p}{\rr}h)}^{2}.\label{eq:Psi}
\end{equation}
This construction is specifically designed with two cancellations
in mind, as the following lemma shows.
\begin{lem}
\label{lem:paired-low} If $g$ and $h$ have disjoint supports and
$\normzero g+\normzero h\le k$, then 
\begin{gather}
\Psi_{p,q}(g,h)=(p^{2}-q^{2})\normtwo{Ag}^{2}+\frac{2pq(p+q)}{\rr}\ip{Ag}{Ah},\label{eq:Psi-expand}\\
\ge[(p^{2}-q^{2})-(p^{2}+q^{2})\delta]\normtwo g^{2}-\frac{2p^{2}q^{2}\delta}{\rr^{2}}\normtwo h^{2}.\label{eq:Psi-rip}
\end{gather}
\end{lem}
\begin{IEEEproof}
Expanding (\ref{eq:Psi}) cancels the two $\normtwo{Ah}^{2}$ terms
and gives (\ref{eq:Psi-expand}). For (\ref{eq:Psi-rip}), apply the
lower RIP endpoint to $g+(q/\rr)h$ and the upper endpoint to $g-(p/\rr)h$.
 Disjoint support gives $\normtwo{g+\tfrac{q}{\rr}h}^{2}=\normtwo g^{2}+\frac{q^{2}}{\rr^{2}}\normtwo h^{2}$
and the analogous identity with $p$.
\end{IEEEproof}
We now package the RIP comparisons into three averaged quantities\begin{subequations}\label{eq:XYZ}
\begin{align}
\mathsf{X}:={} & \frac{\rr-\beta}{\alpha}\E\Psi_{\alpha,\beta}(\sigma_{\bm{T}},\bm{u})+\frac{\rr-\alpha}{\beta}\E\Psi_{\beta,\alpha}(\sigma_{\bm{S}},\bm{v}),\label{eq:X}\\
\mathsf{Y}:={} & \E\!\left[\normtwo{A(\sigma_{\bm{T}}+\sigma_{\bm{S}})}^{2}\right.\nonumber \\[-1mm]
 & \left.\hspace{5mm}-\frac{1-t}{\alpha\beta}\normtwo{A(\beta\sigma_{\bm{T}}-\alpha\sigma_{\bm{S}})}^{2}\,\middle|\,\bm{T}\cap\bm{S}=\varnothing\right],\label{eq:Y}\\
\mathsf{Z}:={} & \E\!\left[\normtwo{A(\sigma_{\head}+(t-1)\bm{w})}^{2}\right.\nonumber \\[-1mm]
 & \left.\hspace{16mm}-(t-1)^{2}\normtwo{A(\sigma_{\head}-\bm{w})}^{2}\right],\label{eq:Z}\\
\vartheta:={} & (\alpha+\beta)^{2}-2\alpha\beta(4-t).\label{eq:vartheta}
\end{align}
\end{subequations}The common quantity $\cX$ couples randomized head
and tail pieces. For $t<1$, $\cY$ reconstructs the measured energy
of the full head from two disjoint head subsets. For $t\ge1$, $\cZ$
uses the additional $k-\rr$ tail budget. 

We will compare $\cX$ against $\cY$ when $0<t<1$, and against $\mathsf{Z}$
for $1\le t<4/3$. Every vector appearing in the definitions of $\mathsf{X}$,
$\mathsf{Y}$, and $\mathsf{Z}$ is $k$-sparse, since the sampled
head and tail supports are disjoint. 
\begin{lem}
\label{lem:balanced-identities}The quantities defined in (\ref{eq:XYZ})
satisfy \begin{subequations}\label{eq:balanced-identities-all}
\begin{align}
\mathsf{X}={} & t\vartheta\normtwo{A\sigma_{\head}}^{2}+2\alpha\beta t(2-t)\ip{A\sigma_{\head}}{A\sigma},\label{eq:X-global}\\
\mathsf{Y}={} & t^{2}\normtwo{A\sigma_{\head}}^{2}\qquad\text{for }0<t<1,\label{eq:Y-global}\\
\mathsf{Z}={} & t(4-3t)\normtwo{A\sigma_{\head}}^{2}\nonumber \\[-1mm]
 & +2t(t-1)\ip{A\sigma_{\head}}{A\sigma}\qquad\text{for }1\le t<4/3.\label{eq:Z-global}
\end{align}
\end{subequations}
\end{lem}
\begin{IEEEproof}
Insert (\ref{eq:Psi-expand}) into (\ref{eq:X}). The two head second
moments satisfy 
\begin{align}
 & \frac{\rr-\beta}{\alpha}\E\normtwo{A\sigma_{\bm{T}}}^{2}-\frac{\rr-\alpha}{\beta}\E\normtwo{A\sigma_{\bm{S}}}^{2}\nonumber \\[-1mm]
 & \hspace{30mm}=\frac{\alpha-\beta}{\rr}\normtwo{A\sigma_{\head}}^{2}.\label{eq:weighted-head-difference}
\end{align}
Indeed, substituting (\ref{eq:uniform-second}) and its $\beta$ analog
cancels the separate sum $\sum_{j\le\rr}\sigma_{j}^{2}\normtwo{A\e_{j}}^{2}$.
Independence, the uniform head means, and the tail means give 
\begin{align*}
\E[\ip{A\sigma_{\bm{T}}}{A\bm{u}}] & =\frac{\alpha}{\rr}\ip{A\sigma_{\head}}{A\sigma_{\tail}},\\
\E[\ip{A\sigma_{\bm{S}}}{A\bm{v}}] & =\frac{\beta}{\rr}\ip{A\sigma_{\head}}{A\sigma_{\tail}}.
\end{align*}
 Consequently, 
\begin{align*}
\mathsf{X}={} & \frac{(\alpha^{2}-\beta^{2})(\alpha-\beta)}{\rr}\normtwo{A\sigma_{\head}}^{2}\\[-1mm]
 & +\frac{2\alpha\beta(\alpha+\beta)(2\rr-\alpha-\beta)}{\rr^{2}}\ip{A\sigma_{\head}}{A\sigma_{\tail}}.
\end{align*}
Using $\alpha+\beta=t\rr$ and $\langle A\sigma_{\head},A\sigma_{\tail}\rangle=\langle A\sigma_{\head},A\sigma\rangle-\|A\sigma_{\head}\|_{2}^{2}$,
together with
\[
(\alpha-\beta)^{2}-2\alpha\beta(2-t)=(\alpha+\beta)^{2}-2\alpha\beta(4-t)=\vartheta
\]
proves (\ref{eq:X-global}). For $0<t<1$, consider the coefficient
of each term in the expansion of $\mathsf{Y}$.  A fixed diagonal
term belongs to $\bm{T}\cup\bm{S}$ with probability $t$. Its expected
squared signed coefficient in $\beta\sigma_{\bm{T}}-\alpha\sigma_{\bm{S}}$
is $\alpha\beta t$.  Hence its coefficient in $\mathsf{Y}$ is 
\[
t-\frac{1-t}{\alpha\beta}\alpha\beta t=t^{2}.
\]
For $i\ne j$, the first squared norm contributes $k(k-1)/[\rr(\rr-1)]$,
while the expected signed coefficient product in the second is $-\alpha\beta k/[\rr(\rr-1)]$.
Therefore every ordered cross term also has coefficient 
\[
\frac{k(k-1)}{\rr(\rr-1)}+\frac{\rr-k}{\alpha\beta\rr}\frac{\alpha\beta k}{\rr(\rr-1)}=t^{2},
\]
which proves (\ref{eq:Y-global}). For $1<t<4/3$, expand $\mathsf{Z}$.
 The $\normtwo{A\bm{w}}^{2}$ terms cancel and $\E[\bm{w}]=\sigma_{\tail}$,
so 
\begin{align*}
\mathsf{Z}={} & [1-(t-1)^{2}]\normtwo{A\sigma_{\head}}^{2}\\[-1mm]
 & +2t(t-1)\ip{A\sigma_{\head}}{A\sigma_{\tail}}.
\end{align*}
Substituting $\sigma_{\tail}=\sigma-\sigma_{\head}$ and rearranging
gives (\ref{eq:Z-global}).  At $t=1$, the identity follows directly
from $\bm{w}=0$.
\end{IEEEproof}
The next lemma applies RIP to obtain upper- and lower-bounds. They
are exactly Zhang--Li's \cite[Eqns 16--18]{ZhangLi2018}.
\begin{lem}[Low-order RIP comparison]
\label{lem:balanced-rip} One has \begin{subequations}\label{eq:balanced-rip-all}
\begin{align}
\mathsf{X}\ge{} & \Bigl(\bigl[(\alpha+\beta)^{2}-4\alpha\beta\bigr]t\nonumber \\[-1mm]
 & \quad-\bigl[(\alpha+\beta)^{2}-2\alpha\beta\bigr](2-t)\delta\Bigr)\mathsf{H}^{2}\nonumber \\[-1mm]
 & -2\alpha\beta(2-t)\delta\mathsf{T}^{2},\label{eq:X-bound}\\
\mathsf{Y}\ge{} & t[t-(2-t)\delta]\mathsf{H}^{2},\qquad0<t<1,\label{eq:Y-bound}\\
\mathsf{Z}\ge{} & [t(2-t)-(t^{2}-2t+2)\delta]\mathsf{H}^{2}\nonumber \\[-1mm]
 & -2(t-1)\delta\mathsf{T}^{2},\qquad1\le t<4/3.\label{eq:Z-bound}
\end{align}
\end{subequations}
\end{lem}
\begin{IEEEproof}
Apply (\ref{eq:Psi-rip}) inside the two expectations defining $\mathsf{X}$.
 The weighted head moments are 
\begin{align*}
\frac{\rr-\beta}{\alpha}\E\normtwo{\sigma_{\bm{T}}}^{2} & =\frac{\rr-\beta}{\rr}\mathsf{H}^{2},\\
\frac{\rr-\alpha}{\beta}\E\normtwo{\sigma_{\bm{S}}}^{2} & =\frac{\rr-\alpha}{\rr}\mathsf{H}^{2}.
\end{align*}
 The part independent of $\delta$ is therefore 
\[
(\alpha^{2}-\beta^{2})\frac{\alpha-\beta}{\rr}\mathsf{H}^{2}=\bigl[(\alpha+\beta)^{2}-4\alpha\beta\bigr]t\mathsf{H}^{2},
\]
 while the $\delta$-weighted head part is 
\[
-(\alpha^{2}+\beta^{2})\frac{2\rr-\alpha-\beta}{\rr}\delta\mathsf{H}^{2}=-\bigl[(\alpha+\beta)^{2}-2\alpha\beta\bigr](2-t)\delta\mathsf{H}^{2}.
\]
 Using (\ref{eq:uv-moments}), the two tail terms are bounded below
by 
\begin{align*}
 & -\frac{2\alpha\beta(\rr-\beta)}{\rr}\delta\mathsf{T}^{2}-\frac{2\alpha\beta(\rr-\alpha)}{\rr}\delta\mathsf{T}^{2}\\
 & \hspace{30mm}=-2\alpha\beta(2-t)\delta\mathsf{T}^{2}.
\end{align*}
This proves (\ref{eq:X-bound}). For $0<t<1$, lower RIP on the first
norm in $\mathsf{Y}$ and upper RIP on the subtracted norm give 
\begin{align*}
\mathsf{Y} & \ge(1-\delta)t\mathsf{H}^{2}-(1+\delta)t(1-t)\mathsf{H}^{2}\\
 & =t[t-(2-t)\delta]\mathsf{H}^{2},
\end{align*}
where (\ref{eq:disjoint-plus})--(\ref{eq:disjoint-minus}) were
used. For $1\le t<4/3$, lower RIP on the first norm in $\mathsf{Z}$
and upper RIP on the second give 
\begin{align*}
\mathsf{Z}\ge{} & [1-\delta-(t-1)^{2}(1+\delta)]\mathsf{H}^{2}\\[-1mm]
 & -2\delta(t-1)^{2}\E\normtwo{\bm{w}}^{2}.
\end{align*}
The head coefficient equals $t(2-t)-(t^{2}-2t+2)\delta$.  Use (\ref{eq:w-moment});
at $t=1$ the tail term vanishes.
\end{IEEEproof}
We extract the following critical conclusion from the Zhang--Li construction,
which also implicitly underlies the proof of their main result \cite[Thm~1.1]{ZhangLi2018}.
\begin{lem}[Key low-order inequality]
\label{lem:balanced-low} For $0<t<4/3$, 
\begin{equation}
[t-(3-t)\delta]\mathsf{H}^{2}\le t\ip{A\sigma_{\head}}{A\sigma}+\delta\mathsf{T}^{2}.\label{eq:balanced-low}
\end{equation}
\end{lem}
\begin{IEEEproof}
Choose $\alpha=\lceil k/2\rceil$, $\beta=\lfloor k/2\rfloor$ so
that $4\alpha\beta\ge k^{2}-1$. Hence, throughout $0<t<4/3$, 
\begin{align*}
\vartheta=(\alpha+\beta)^{2}-2\alpha\beta(4-t) & \le k^{2}-{\textstyle \frac{k^{2}-1}{2}}(4-t)<0,
\end{align*}
where the strict inequality uses $k\ge2$ and $t<4/3$. 

Suppose first that $0<t<1$. Equations~(\ref{eq:X-global}) and (\ref{eq:Y-global})
give 
\[
t\mathsf{X}-2\alpha\beta t^{2}(2-t)\ip{A\sigma_{\head}}{A\sigma}=\vartheta\mathsf{Y}.
\]
Since $t>0$, (\ref{eq:X-bound}) gives a lower bound for the left-hand
side. Since $\vartheta<0$, the lower bound (\ref{eq:Y-bound}) gives
an upper bound for the right-hand side. Substituting these two bounds
and collecting terms yields 
\begin{align*}
2\alpha\beta t(2-t)\Bigl( & [t-(3-t)\delta]\mathsf{H}^{2}\\[-1mm]
 & -t\ip{A\sigma_{\head}}{A\sigma}-\delta\mathsf{T}^{2}\Bigr)\le0.
\end{align*}
The prefactor is positive, which proves (\ref{eq:balanced-low}) for
$0<t<1$.

Now suppose $1\le t<4/3$. Equations~(\ref{eq:X-global}) and (\ref{eq:Z-global})
give 
\[
(4-3t)\mathsf{X}-2t^{3}[\alpha\beta-(t-1)\rr^{2}]\ip{A\sigma_{\head}}{A\sigma}=\vartheta\mathsf{Z}.
\]
Here $4-3t>0$, so (\ref{eq:X-bound}) again gives a lower bound for
the left-hand side. Since $\vartheta<0$, the lower bound (\ref{eq:Z-bound})
gives an upper bound for the right-hand side. Substitution and collection
of terms give 
\begin{align*}
 & 2t^{2}[\alpha\beta-(t-1)\rr^{2}]\\[-1mm]
 & \quad\times\Bigl([t-(3-t)\delta]\mathsf{H}^{2}-t\ip{A\sigma_{\head}}{A\sigma}-\delta\mathsf{T}^{2}\Bigr)\le0.
\end{align*}
The remaining prefactor is positive because 
\begin{align*}
\alpha\beta-(t-1)\rr^{2} & \ge{\textstyle \frac{k^{2}-1}{4}}-(t-1)\rr^{2}={\textstyle \frac{1}{4}}[(2-t)^{2}\rr^{2}-1]>0.
\end{align*}
using $1\le t<4/3$ and $\rr\ge2$. Therefore the common bracket is
nonpositive, proving (\ref{eq:balanced-low}) for $1\le t<4/3$. 
\end{IEEEproof}
Before proceeding, we first upper-bound on the cross term in (\ref{eq:balanced-low}). 
\begin{claim}
\label{claim:mixed} For $0<t<4/3$, 
\begin{equation}
t\ip{A\sigma_{\head}}{A\sigma}\le\max\{t,\sqrt{t}\}\sqrt{1+\delta}\,\mathsf{H}\normtwo{A\sigma}.\label{eq:mixed}
\end{equation}
\end{claim}
\begin{IEEEproof}
If $1\le t<4/3$, then $\normzero{\sigma_{\head}}=\rr\le k$. Cauchy--Schwarz
and $(\delta,k)$-RIP gives 
\[
t\ip{A\sigma_{\head}}{A\sigma}\le t\normtwo{A\sigma_{\head}}\normtwo{A\sigma}\le t(\sqrt{1+\delta}\,\mathsf{H})\normtwo{A\sigma}.
\]
If $0<t<1$, then $\normzero{\sigma_{\head}}=\rr>k$, but \lemref{ric-comparison}
shows that $A$ nevertheless satisfies $(\delta',\rr)$-RIP with $\delta'=(2/t-1)\delta$.
Therefore 
\[
t\normtwo{A\sigma_{\head}}\le t\sqrt{1+(2/t-1)\delta}\,\mathsf{H}\le\sqrt{t(1+\delta)}\,\mathsf{H}
\]
because $(1+\delta)-t[1+(2/t-1)\delta]=(1-t)(1-\delta)\ge0$.
\end{IEEEproof}
\begin{IEEEproof}[Proof of Proposition~\ref{prop:HTR} for $0<t<4/3$]
Combining Lemma~\ref{lem:balanced-low} and Claim~\ref{claim:mixed},
then dividing by $t-(3-t)\delta>0$, gives 
\begin{equation}
\mathsf{H}^{2}\le\tau_{\delta,t}\mathsf{H}\normtwo{A\sigma}+\rho_{\delta,t}^{2}\mathsf{T}^{2}.\label{eq:low-quadratic}
\end{equation}
 If $\mathsf{H}\le\rho_{\delta,t}\mathsf{T}$, then (\ref{eq:HTR})
is immediate.  Otherwise, 
\[
(\mathsf{H}-\rho_{\delta,t}\mathsf{T})(\mathsf{H}+\rho_{\delta,t}\mathsf{T})\le\tau_{\delta,t}\mathsf{H}\normtwo{A\sigma}.
\]
 Since $\mathsf{H}+\rho_{\delta,t}\mathsf{T}\ge\mathsf{H}>0$, division
gives 
\[
\mathsf{H}-\rho_{\delta,t}\mathsf{T}\le\tau_{\delta,t}\normtwo{A\sigma},
\]
 which is (\ref{eq:HTR}).  This completes Proposition~\ref{prop:HTR}.
\end{IEEEproof}

\section{Conclusion}

We establish the sharp recovery threshold $\delta<\sharpdelta(t)$
for statistical recovery at global minima of the rank-restricted matrix
LASSO (with $t=k/\rr$) and sparsity-restricted vector LASSO (with
$t=k/\ss$). The proof confirms that penalized recovery does not worsen
the sharp RIP threshold previously established for constrained recovery.
The key proof idea is to retain the violation of the head-tail inequality
with a defect term, so that $\cT\le\cH+\cD$ additively, instead of
$\cT\le c\cH$ multiplicatively for some $c>1$. Combining this with
a sharp head estimate $\cH\le\rho_{\delta,t}\cT+\text{noise}$ yields
recovery whenever the feedback gain $\rho_{\delta,t}<1$, which coincides
exactly with the sharp RIP threshold condition $\delta<\sharpdelta(t)$. 

The resulting picture separates two effects that are often conflated
in low-rank optimization. Increasing the search rank can substantially
complicate the nonconvex landscape and strengthen the assumptions
needed to control every second-order critical point. At global minima,
by contrast, the sharp recovery threshold depends only on the ground-truth
rank and is independent of the search rank, provided the search space
contains the ground truth. Rank overparameterization is therefore
not itself a statistical obstruction at global optimality. Its cost
appears instead in the optimization landscape and in the difficulty
of finding a global minimizer.

\section*{Acknowledgments}

The author thanks Joshua Agterberg, Andrew D. McRae, and Sabrina Zielinski
for helpful discussions that motivated this work. Financial support
was provided by NSF CAREER Award ECCS-2047462 and ONR Award N00014-24-1-2671. 

\bibliographystyle{IEEEtran}
\bibliography{sharp_rip_matrix_vector_lasso}

\end{document}